%% file: graphood_conference.tex
\documentclass{article} 
\usepackage{graphood_conference,times}

\input{math_commands.tex}

\usepackage[table,dvipsnames]{xcolor}
\usepackage{pifont}
\usepackage{hyperref}
\usepackage{url}
\usepackage{amsmath}
\usepackage{amssymb}
\usepackage{amsthm}
\usepackage{mathtools}
\usepackage{booktabs}
\usepackage{algorithm}
\usepackage{algpseudocode}
\usepackage{graphicx}
\usepackage{wrapfig}
\usepackage{caption}
\usepackage{enumitem}
\usepackage{xspace}
\usepackage{tcolorbox}
\tcbuselibrary{breakable}
\usepackage[capitalise,noabbrev]{cleveref}
\crefname{theorem}{Theorem}{Theorems}
\crefname{lemma}{Lemma}{Lemmas}
\crefname{proposition}{Proposition}{Propositions}
\crefname{corollary}{Corollary}{Corollaries}
\makeatletter
\renewcommand{\paragraph}{\@startsection{paragraph}{4}{\z@}{0.5ex plus 0.2ex minus .1ex}{-1em}{\normalsize\bf}}
\makeatother

\definecolor{paperaccent}{HTML}{6A3D9A}  
\definecolor{paperobs}{HTML}{1F4E79}     
\definecolor{papercite}{HTML}{168AAD}    
\definecolor{paperbox}{HTML}{F3F7FA}     
\definecolor{tabhead}{HTML}{E4EDF6}      
\definecolor{tabalt}{HTML}{F5F5F5}       
\definecolor{tabours}{HTML}{EDE7F6}      
\hypersetup{colorlinks=true,linkcolor=paperaccent,citecolor=papercite,urlcolor=papercite}
\newcommand{\pointmark}[1]{\textcolor{paperaccent}{\ding{#1}}}
\newcommand{\key}[1]{\textbf{#1}}
\newcommand{\obs}[1]{\textcolor{paperobs}{\textbf{#1}}}
\newcommand{\pcref}[1]{\textcolor{paperaccent}{\cref{#1}}}
\newtcolorbox{corebox}{
  colback=paperbox, colframe=paperaccent!55!black,
  boxrule=0.45pt, arc=1.5pt,
  left=5pt, right=5pt, top=4pt, bottom=4pt,
  breakable
}

\theoremstyle{plain}
\newtheorem{theorem}{Theorem}[section]
\newtheorem{proposition}[theorem]{Proposition}
\newtheorem{lemma}[theorem]{Lemma}

\theoremstyle{definition}

\theoremstyle{remark}

\newcommand{\ours}{\textsc{Scarce}\xspace}

\newcommand{\mbf}[1]{\mathbf{#1}}

\title{Implicit Feature Reweighting for Label-Free Graph Generalization}

\author{%
  Tianjun Yao\textsuperscript{1}\quad
  Chen Xi\textsuperscript{2,1}\quad
  Xinyang Li\textsuperscript{3,1}\quad
  Lixian Zhang\textsuperscript{4,1}\\[3pt]
  \normalfont\textsuperscript{1}Shenzhen Loop Area Institute\\
  \normalfont\textsuperscript{2}Guangzhou University\\
  \normalfont\textsuperscript{3}North China Electric Power University\\
  \normalfont\textsuperscript{4}University of Science and Technology of China
}
\hypersetup{pdfauthor={Tianjun Yao, Chen Xi, Xinyang Li, Lixian Zhang}}

\iclrfinalcopy
\begin{document}

\maketitle
\lhead{Preprint}
\suppressfloats[t]

\begin{abstract}
Graph-level out-of-distribution (OOD) generalization is critical for deploying graph neural networks in real-world applications where distribution shifts are inevitable. Existing approaches typically require class labels during representation learning to decide which substructures are predictive, limiting their applicability. We propose \ours, a label-free representation-learning framework that yields distribution-shift-robust graph representations from an environment partition of the training graphs alone, whether that partition comes from dataset metadata or, when none is available, from clustering a self-supervised embedding. Our method operates in two stages: (1)~a variational graph autoencoder is warm-started for graph reconstruction; and (2)~an adversarial confusion mechanism forces the encoder to produce representations from which a discriminator cannot distinguish environments, while preserving graph-reconstruction fidelity. We provide theoretical analysis showing that environment confusion contracts environment-discriminative directions in the representation space, which, under regularized downstream classifiers, is equivalent to imposing a stronger directional penalty on environment-associated features. This feature-reweighting mechanism reduces reliance on spurious correlations without erasing information, offering a label-free, continuous relaxation of last-layer reweighting methods. Experiments on a controlled synthetic model, the GOOD-Motif benchmark, and molecular property prediction under assay, scaffold, and size shifts show that \ours achieves state-of-the-art OOD generalization among label-free methods, using no class labels during representation learning.
\end{abstract}

\input{Sections/intro}

\input{Sections/related_work}

\input{Sections/problem_setting}

\input{Sections/method}

\input{Sections/experiments}

\input{Sections/conclusion}

\bibliography{graphood_conference}
\bibliographystyle{graphood_conference}

\clearpage
\appendix
\begin{center}{\Large\sc Appendix}\end{center}
\vspace{-0.5ex}
\input{Sections/appendix}

\end{document}

%% file: math_commands.tex
\usepackage{amsmath,amsfonts,bm}

\def\eqref#1{equation~\ref{#1}}

\def\1{\bm{1}}

\DeclareMathAlphabet{\mathsfit}{\encodingdefault}{\sfdefault}{m}{sl}
\SetMathAlphabet{\mathsfit}{bold}{\encodingdefault}{\sfdefault}{bx}{n}



%% file: Sections/intro.tex
\section{Introduction}
\label{sec:intro}

Machine learning models are routinely deployed on data whose distribution differs from that of their training set, and their tendency to latch onto features that are predictive only in training makes such shifts a primary cause of failure~\citep{geirhos2020shortcut,gulrajani2020search}. Graph data sharpens the problem: molecules are tested on scaffolds absent from training, and graph size or base-graph structure drifts between collections~\citep{gui2022good,ji2022drugood,li2025out}. A substantial body of work now targets out-of-distribution (OOD) generalization on graphs, typically by extracting an invariant subgraph or enforcing invariance across environments~\citep{wu2022discovering,chen2022ciga,miao2022interpretable,yang2022learning,wu2022handling,gui2024joint}. These methods share a prerequisite that is easy to overlook: they need class labels to decide which substructures are predictive, environment labels to decide which correlations are unstable, or both.

Class labels are the expensive part. In drug discovery a single assay label is the outcome of a wet-lab experiment, and many DrugOOD assays contain only tens of labeled molecules, whereas unlabeled molecular graphs are available by the tens of millions~\citep{ji2022drugood,hu2019strategies,rong2020self}. Environment information, by contrast, is cheap but rough: a benchmark may ship a domain index, a scaffold can be computed although which scaffolds form an environment is a modeling decision, and for proteins or social networks the partition has to be inferred from the data itself~\citep{lin2022zin,yao2024empowering}. A method that learns shift-robust representations from an environment partition alone, with no class labels, and that tolerates a coarse or inferred partition, would therefore cover the common case. This setting remains largely unexplored on graphs: the few existing label-free methods add invariance-inspired augmentations to a contrastive objective~\citep{li2022let,li2024disentangled}, and none establishes what the learned representation actually retains.

\begin{figure}[t]
\centering
\includegraphics[width=0.9\linewidth]{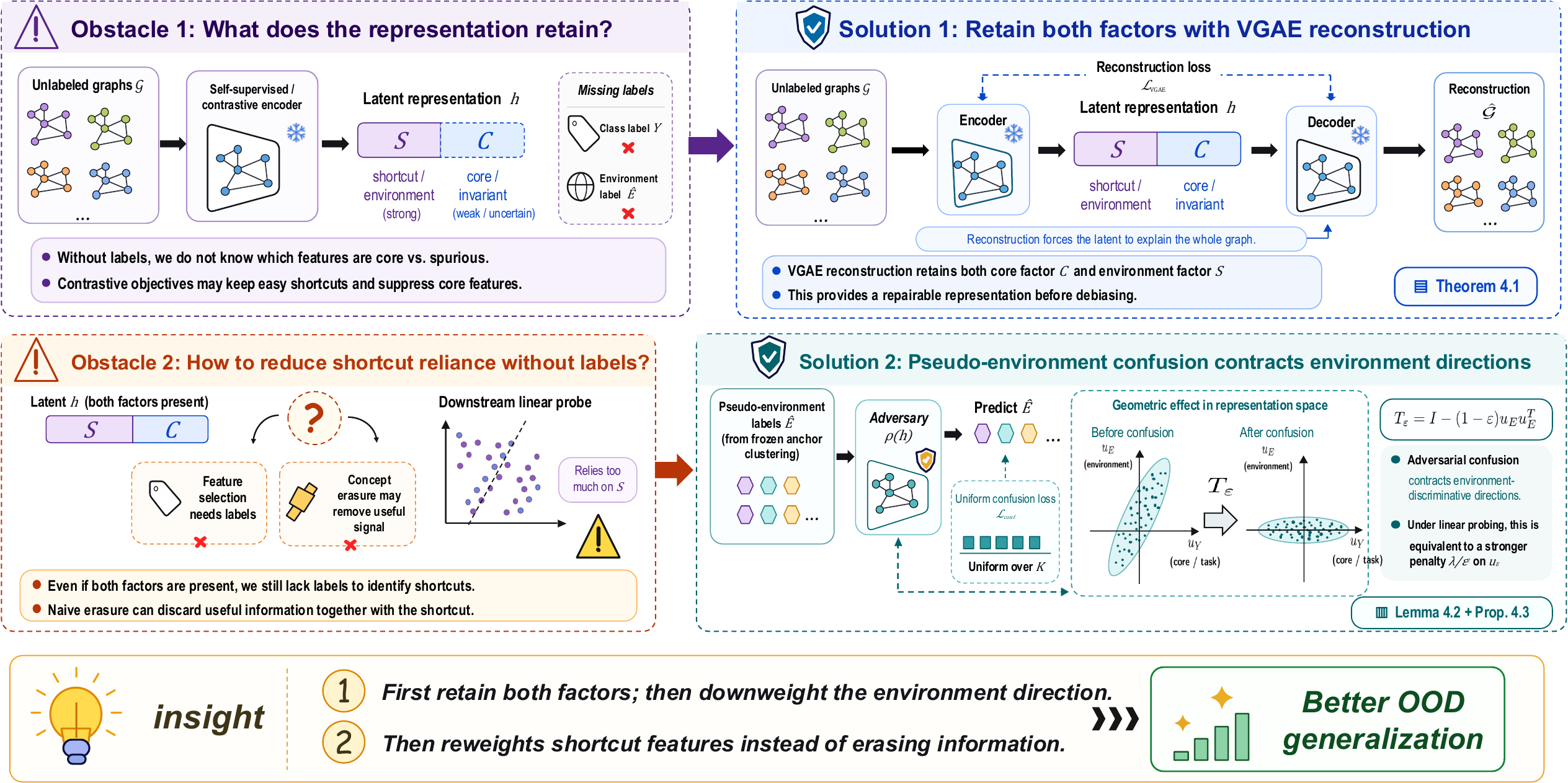}
\caption{Overview of the \ours framework. Given an environment partition of the training graphs (dataset metadata or, when unavailable, clusters of a frozen self-supervised embedding), (1)~a VGAE is warm-started for graph reconstruction and (2)~an adversarial confusion mechanism forces the encoder to produce representations indistinguishable across environments while preserving reconstruction fidelity.}
\label{fig:framework}
\end{figure}

Removing class supervision raises two obstacles, and each points to one of our contributions. \key{The first concerns what the representation contains.} Without labels one can neither tell the encoder which features to keep nor check afterwards which it kept, and a contrastive objective offers no guarantee: it is minimized by whichever features most easily discriminate instances and suppresses the rest~\citep{robinson2021can}, and a representation that has discarded the core features cannot be repaired downstream. Likelihood objectives behave differently: because the latent must explain the whole input, they keep core features alongside shortcuts~\citep{yang2022chroma,li2025generative,muller2024shortcut,qiang2025out}. We show that this holds for graph reconstruction without any auxiliary variable: for a variational graph autoencoder~\citep{kipf2016variational}, reconstruction alone forces the latent to retain \emph{both} environment-specific and core factors (\pcref{thm:retention}), and both remain linearly accessible. \key{The second obstacle concerns what to do with them.} No class label says which factor is spurious, and erasing environment information outright discards useful signal along with the shortcut~\citep{belrose2023leace,lin2024spurious}. We show that neither selection nor erasure is needed: it suffices to \emph{shrink} the environment-associated directions, because any $\ell_2$-regularized downstream classifier then finds them expensive to use and shifts its weight onto the core features. We realize this shrinkage with an adversarial confusion loss against the training environments, which asks nothing of the partition beyond that it separate the spurious factor, and call the framework \ours (\textbf{S}hortcut \textbf{C}ontraction via \textbf{A}dversarial \textbf{R}epresentation \textbf{C}ompression for \textbf{E}nvironments).

Our contributions are as follows.
\begin{itemize}[leftmargin=1.4em,itemsep=1pt,topsep=2pt,label={}]
\item[\pointmark{172}] \key{Retention by reconstruction.} We prove that VGAE reconstruction alone forces the latent to retain both environment-specific and core features (\pcref{thm:retention}), and verify that both are linearly accessible.
\item[\pointmark{173}] \key{Contraction as a ridge penalty.} We propose a confusion loss that scales down environment-discriminative directions rather than erasing them, prove that this contraction (\pcref{lem:contraction}) is exactly equivalent to an anisotropic ridge penalty on the downstream probe (\pcref{prop:ridge}), and characterize exactly when it helps.
\item[\pointmark{174}] \key{Empirical validation.} Empirically, \ours attains state-of-the-art label-free OOD accuracy on GOOD-Motif and is best in 17 of 18 DrugOOD settings, where no environment annotation exists and clustered pseudo-environments suffice; a controlled synthetic study confirms the two preconditions identified by the theory.
\end{itemize}

\pcref{sec:related_work} positions \ours against supervised and contrastive alternatives, \pcref{sec:problem_setting} and \pcref{sec:method} state the setting and the two guarantees, and \pcref{sec:experiments} tests the mechanism from a synthetic model up to DrugOOD.

%% file: Sections/related_work.tex
\section{Related Work}
\label{sec:related_work}

\paragraph{OOD generalization with labeled supervision.}
Most OOD methods prevent a model from relying on features whose correlation with the label is unstable~\citep{geirhos2020shortcut,gulrajani2020search}: invariant risk minimization~\citep{arjovsky2020invariant}, group-robust optimization~\citep{sagawa2020distributionally}, and domain-adversarial alignment~\citep{ganin2016domain} use environment or group annotations, and environment inference~\citep{creager2021environment} substitutes class labels for them. \key{In every case some form of supervision decides which correlations are unstable.}

\paragraph{Graph OOD generalization.}
Graph OOD methods either extract an invariant subgraph whose relation to the label is stable~\citep{wu2022discovering,chen2022ciga,miao2022interpretable,li2022learning,yang2022learning,chen2023does} or enforce invariance across given or generated environments~\citep{wu2022handling,gui2024joint}; in both cases a class label decides which substructures are predictive. \key{\ours uses no class labels during representation learning}; it needs only an environment partition, which may be given or inferred, and it does not select subgraphs.

\paragraph{Unsupervised graph representation learning.}
Unsupervised graph representation learning is mature~\citep{sun2019infograph,you2020graph,hassani2020contrastive,hou2022graphmae}, and unsupervised domain generalization has been formulated for images~\citep{zhang2022towards}, but their intersection on graphs is thin. RGCL~\citep{li2022let} adds rationale-inspired augmentations to a contrastive objective, and OOD-GCL~\citep{li2024disentangled}, the closest work to ours, learns a channel-wise disentangled contrastive encoder and penalizes each factor's contrastive loss for varying across inferred environments. Both are contrastive at the core, which suppresses features not needed to tell instances apart~\citep{robinson2021can}, and both make environment information invariant; \ours instead trains by reconstruction, which provably retains core and environment factors alike (\pcref{thm:retention}), and contracts environment directions rather than erasing them (\pcref{sec:theory}). When no environment annotation is available we borrow the construction of \citet{yao2024empowering}, who cluster a frozen self-supervised embedding in place of environment labels but then train a supervised graph OOD model; we use the same pseudo-environments inside a label-free objective (\pcref{sec:real_world}).

\paragraph{Feature reweighting and concept erasure.}
Two lines of work act on a representation after it is learned. Last-layer retraining~\citep{kirichenko2023last} re-fits the classifier on group-balanced data, reweighting core over spurious features at the cost of group labels. Linear concept erasure~\citep{ravfogel2022linear,belrose2023leace} removes every linearly decodable trace of a protected attribute, discarding task-useful environment signal as well~\citep{lin2024spurious}, and the probes that certify removal are themselves unreliable~\citep{kumar2024probing}. \key{\ours sits between the two:} it reweights rather than erases, acts in representation space rather than on the classifier, and, since the contracted direction is chosen by an adversary against environments, needs neither group-balanced labeled data nor a protected attribute.

%% file: Sections/problem_setting.tex
\section{Problem Setting}
\label{sec:problem_setting}

We consider graph-level \emph{unsupervised domain generalization}: learning graph representations that are robust to distribution shifts, without access to class labels during representation learning.

Let the training set contain $N$ graphs $\{G_1, G_2, \dots, G_N\}$, where each graph $G_i = (\mathbf{X}_i, \mathbf{A}_i)$ consists of a node feature matrix $\mathbf{X}_i \in \mathbb{R}^{n_i \times d_{\mathrm{in}}}$ and an adjacency structure $\mathbf{A}_i$. The training graphs are drawn from a mixture of environments (domains), and we seek representations that generalize to unseen test environments. Each training graph carries an environment index $e_i \in \{1, \dots, K\}$. The partition may be supplied by the dataset (a domain index) or, absent such metadata, inferred without labels by clustering a self-supervised embedding (\pcref{sec:real_world}); nothing below depends on which.

\paragraph{Latent generative factors.} We posit that each graph is generated from three latent factors through an unknown map $g$,
\begin{equation}
G = g(C, S, \xi),
\label{eq:generative_model}
\end{equation}
where $C$ is the \emph{core} (invariant) factor that determines the label $Y$ in every environment, $S$ is the \emph{environment} (spurious) factor whose correlation with $Y$ varies across environments $e$, and $\xi$ collects residual variation independent of $(C, S)$. A predictor that relies on $S$ can succeed in training environments where the $S$--$Y$ correlation is strong yet fail on test environments where it weakens or reverses, so $C$ is the only reliable predictor of $Y$. We write $Z = f_\theta(G)$ for the latent representation produced by a graph encoder $f_\theta$, viewed as a random variable, and $\mathbf{h}_G \in \mathbb{R}^d$ for its value on a given graph $G$. For a reconstruction-based encoder--decoder pair, $L_{\mathrm{rec}} = \mathbb{E}\!\left[-\log p_\psi(G \mid Z)\right]$ denotes the expected reconstruction loss (negative log-likelihood, in nats), i.e., the reconstruction term of the evidence lower bound.

\paragraph{Encoder.} We encode graphs with message-passing GNNs~\citep{xu2019powerful,kipf2017semi}: $L$ layers of neighborhood aggregation produce node embeddings $\mathbf{h}_v^{(L)}$, and the graph-level latent is obtained by sum pooling, $\mathbf{h}_G = \sum_{v \in G} \mathbf{h}_v^{(L)}$. GIN is our default backbone; we also report GCN.

\paragraph{Objective.} Let $e \in \mathcal{E}$ index the environments. The goal is to learn a graph encoder $f_\theta$ whose representations, when paired with a downstream classifier $c_\phi$, minimize the worst-case risk across all environments including unseen ones:
\begin{equation}
\min_\theta \;\sup_{e \in \mathcal{E}_{\mathrm{all}}} \; R^e(\theta), \qquad R^e(\theta) = \mathbb{E}_{(G, Y) \sim P^e}\!\left[\ell\!\left(c_\phi(f_\theta(G)),\, Y\right)\right].
\label{eq:ood_risk}
\end{equation}
We call this setting \emph{label-free}: no class label $Y$ is available during representation learning, the environment index $e$ is the only side information, and it may itself be inferred. Class labels are used \emph{only} for fitting a linear probe at evaluation time. This constraint rules out existing invariant learning methods~\citep{arjovsky2020invariant,wu2022handling,chen2022ciga}, which need class labels to decide which correlations are unstable.

%% file: Sections/method.tex
\section{Method}
\label{sec:method}

We introduce \ours, a framework for graph-level unsupervised domain generalization. The key insight is that a standard VGAE learns representations encoding \emph{all} predictive variation, including environment-specific features that may serve as spurious shortcuts. \ours addresses this by augmenting the VGAE objective with an adversarial confusion mechanism that selectively suppresses environment-discriminative directions in the representation, while preserving graph-reconstruction fidelity. An overview is presented in \pcref{fig:framework} and \pcref{alg:guide_conf}.

\subsection{What Features Does a VGAE Retain?}
\label{sec:vgae_features}

Our method relies on a VGAE whose latent encodes \emph{both} generative factors: without $S$ there is nothing for the adversary to contract, and without $C$ contraction cannot help. Identifiability results for VAE-style models~\citep{khemakhem2020variational,zheng2022identifiability,kong2023partial,sun2021recovering,wang2022causal} and the shortcut-robustness of generative classifiers~\citep{li2025generative,qiang2025out,muller2024shortcut} both trace factor retention to the reconstruction objective. We ask whether VGAE reconstruction \emph{forces} retention of both factors, and answer affirmatively.

\begin{theorem}[Reconstruction forces retention of both factors]
\label[theorem]{thm:retention}
Let $G = g(C, S, \xi)$ as in \cref{eq:generative_model} and let $Z = f_\theta(G)$ be the encoder latent. Assume \textup{(A1)} the generative map $g$ is injective in each argument with the others held fixed, and \textup{(A2)} $G$ is discrete with $H(S \mid C,\xi)$ and $H(C \mid S,\xi)$ finite and positive. Then any encoder--decoder pair attaining reconstruction loss $L_{\mathrm{rec}}$ satisfies \emph{simultaneously}
\begin{equation}
I(Z; S \mid C, \xi) \;\ge\; H(S \mid C, \xi) - L_{\mathrm{rec}},
\qquad
I(Z; C \mid S, \xi) \;\ge\; H(C \mid S, \xi) - L_{\mathrm{rec}},
\label{eq:retention_bounds}
\end{equation}
both bounds being strictly positive whenever $L_{\mathrm{rec}} < \min\{H(S \mid C,\xi),\, H(C \mid S,\xi)\}$.
\end{theorem}

The key distinction from discriminative training is that the reconstruction objective imposes \emph{per-factor} constraints whose strength scales with each factor's generative footprint on the observed graph, whereas a classification loss yields only a single aggregate constraint bounded by $H(Y)$, which is satisfiable by an encoder that discards $C$ entirely when $S$ alone predicts the training labels. Importantly, the VGAE does not itself learn a ``shortcut'': it retains $S$ as a generative factor, and the shortcut arises only downstream when a predictor exploits the spurious $S$--$Y$ correlation. \pcref{thm:retention} establishes information-theoretic retention; it does not guarantee linear accessibility or preservation after graph pooling. The full proof is deferred to \textcolor{paperaccent}{Appendix~\ref{app:retention_proof}}.

\begin{figure}[t]
\centering
\begin{minipage}[c]{0.40\linewidth}
\centering
\includegraphics[width=\linewidth]{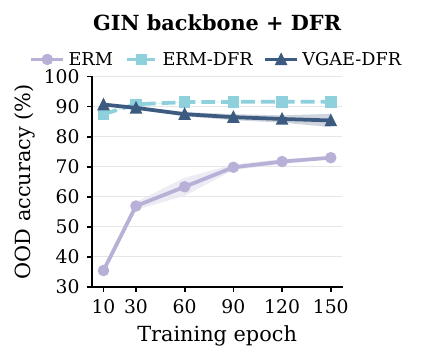}
\caption{DFR on frozen GIN representations across training, GOOD-Motif concept shift (OOD accuracy, mean$\pm$std over 3 seeds; $y$-axis starts at 30\%). ERM improves slowly, yet DFR on the same encoder is already above 87\% at epoch 10; DFR on the unsupervised VGAE starts at 90.7\% and degrades mildly with further training.}
\label{fig:dfr_verification}
\end{minipage}%
\hfill
\begin{minipage}[c]{0.57\linewidth}
\captionof{algorithm}{\ours Training Pipeline}
\label{alg:guide_conf}
\small
\begin{algorithmic}[1]
\Require graphs $\{G_i\}_{i=1}^N$, environment indices $\{e_i\}$, weights $\{w_i\}$
\State Pre-train VGAE on $\mathcal{L}_{\mathrm{VGAE}}$; save warm-start checkpoint
\State Initialize VGAE from warm-start checkpoint; initialize adversary $\rho$
\For{$t = 1, \dots, T$}
    \For{each mini-batch $\mathcal{B}$}
        \State Compute $\{\mathbf{h}_G^{(i)}\}_{i\in\mathcal{B}}$ with the VGAE encoder
        \State \textbf{Adversary step:} freeze encoder; $k_{\mathrm{adv}}$ updates of $\rho$ on $\mathrm{sg}(\mathbf{h}_G)$ with \cref{eq:adv_loss}
        \State \textbf{Encoder step:} freeze $\rho$; update encoder with \cref{eq:enc_total}
    \EndFor
\EndFor
\State \textbf{Evaluation:} freeze encoder; fit a linear probe on training labels
\end{algorithmic}
\end{minipage}
\end{figure}
\paragraph{Empirical verification via feature reweighting.}
\pcref{thm:retention} is information-theoretic; we test whether both factors are also \emph{linearly accessible} with Deep Feature Reweighting (DFR;~\citealt{kirichenko2023last}), which freezes the encoder and refits a sparse linear head on a group-balanced subset, so its OOD accuracy measures the invariant signal the latent exposes once the spurious $S$--$Y$ correlation is removed. We apply DFR to GIN checkpoints of a label-trained ERM encoder and of an unsupervised VGAE on GOOD-Motif concept shift (3 seeds; protocol in \textcolor{paperaccent}{Appendix~\ref{app:dfr_protocol}}). \pcref{fig:dfr_verification} shows two things. First, ERM climbs slowly from $35.4\%$ to $73.0\%$, yet DFR on the \emph{same} frozen encoder is already at $87.5\%$ at epoch $10$: $C$ is linearly accessible early, and the ERM head merely underweights it relative to $S$. Second, DFR on the VGAE, which never sees a label, reaches $90.7\%$ at epoch $10$, on par with the label-trained encoder: reconstruction exposes $C$ in linear form from the outset, and, since the same latent is what a naive probe exploits via $S$, it exposes $S$ as well. With longer VGAE training, DFR accuracy drifts down to $85.4\%$ with growing variance, suggesting that prolonged reconstruction slowly reallocates latent capacity away from the directions a linear head needs. Since DFR probes the pooled $\mathbf{h}_G$ rather than the latent the decoder sees, retention evidently survives pooling as well.

\subsection{Adversarial Environment Confusion}
\label{sec:adv_training}

Since the VGAE latent retains environment-specific features alongside invariant ones (\pcref{sec:vgae_features}), \ours adds an encoder--adversary alternating optimization that discourages reliance on the environment-discriminative directions.

\paragraph{Adversary.} The adversarial classifier $\rho: \mathbb{R}^d \to \mathbb{R}^K$ is a two-layer MLP that maps graph representations to environment logits. In each mini-batch, the encoder is frozen and the adversary is updated for $k_{\mathrm{adv}}$ steps on the \emph{detached} representations $\mathbf{h}_{\mathrm{adv}} = \mathrm{sg}(\mathbf{h}_G)$, minimizing the rebalanced cross-entropy:
\begin{equation}
\mathcal{L}_{\mathrm{adv}} = \frac{\sum_{i} w_i \cdot \mathrm{CE}(\rho(\mathbf{h}_{\mathrm{adv}}^{(i)}),\; e_i)}{\sum_{i} w_i},
\label{eq:adv_loss}
\end{equation}
where $e_i$ are the environment indices of \pcref{sec:problem_setting} and $w_i \propto n_{e_i}^{-\alpha}$ are rebalancing weights that compensate for uneven environment sizes, with $n_{e}$ the number of training graphs in environment $e$ and $\alpha \ge 0$ controlling the strength of rebalancing ($\alpha=0$ gives uniform weights). Rebalancing matters most when the partition is inferred, since clusters are rarely of equal size.

\paragraph{Encoder update via uniform confusion.} The adversary is then frozen, and the encoder is updated. The confusion loss drives the adversary's output toward the uniform distribution $\mathbf{u} = (1/K, \dots, 1/K)$:
\begin{equation}
\mathcal{L}_{\mathrm{conf}} = \frac{\sum_{i} w_i \cdot \left(-\sum_{k=1}^{K} \frac{1}{K} \log \frac{\exp(\rho_k(\mathbf{h}_G^{(i)}))}{\sum_{j}\exp(\rho_j(\mathbf{h}_G^{(i)}))}\right)}{\sum_{i} w_i}.
\label{eq:conf_loss}
\end{equation}
For each sample, this equals $H(U_K, p_\rho) = \log K + \mathrm{KL}(U_K \| p_\rho)$, minimized when the adversary outputs $1/K$ for every environment. The alternating game works as follows: the adversary identifies representation directions that separate environments; the encoder receives gradients that weaken those distinctions; and the reconstruction objective constrains the encoder from discarding all variation.

\paragraph{Encoder total loss.} The encoder loss combines reconstruction and confusion:
\begin{equation}
\mathcal{L}_{\mathrm{enc}} = \mathcal{L}_{\mathrm{VGAE}} + \beta \cdot \mathcal{L}_{\mathrm{conf}},
\label{eq:enc_total}
\end{equation}
where $\beta$ controls the strength of the confusion penalty: $\mathcal{L}_{\mathrm{VGAE}}$ keeps the latent informative (\pcref{sec:vgae_features}) while $\mathcal{L}_{\mathrm{conf}}$ contracts the directions the adversary identifies.

\subsection{Theoretical Analysis: Why Confusion Helps OOD Generalization}
\label{sec:theory}

Erasing environment information discards useful signal together with the shortcut~\citep{belrose2023leace,lin2024spurious}. We show instead that confusion training \emph{contracts} environment-discriminative directions (\pcref{lem:contraction}), and that under a regularized downstream classifier this geometric change is exactly a directional penalty that makes spurious features expensive to use (\pcref{prop:ridge}). Proofs are in \textcolor{paperaccent}{Appendix~\ref{app:proofs}}.

\begin{lemma}[Local contraction]
\label[lemma]{lem:contraction}
Fix a linear $K$-class softmax adversary with logits $z(h)=Wh+b$, probabilities $p(h)=\operatorname{softmax}(z(h))$, and uniform target $q=K^{-1}\mathbf{1}$. Define the centered adversary weights $W_c=CW$ where $C=I-K^{-1}\mathbf{1}\mathbf{1}^\top$. Then a gradient step $h^+ = h - \alpha\beta\,\nabla_h L_{\mathrm{conf}}(h)$ contracts each adversary-discriminative direction $v_j$ (right singular vector of $W_c$ with singular value $s_j$) by a factor $1 - \alpha\beta s_j^2 / K$ to first order, while directions orthogonal to $\operatorname{range}(W_c^\top)$ are unchanged. In particular, when one singular value dominates ($s_1 \gg s_2$), the cumulative effect of confusion training is well approximated by a rank-one contraction along $u_E = v_1$:
\begin{equation}
\mathbf{h}_G^{(\beta)} \approx T_\varepsilon\, \mathbf{h}_G^{(0)}, \qquad T_\varepsilon = I - (1-\varepsilon)\, u_E u_E^\top, \quad 0 < \varepsilon \le 1,
\label{eq:contraction_model}
\end{equation}
where $\varepsilon = 1$ corresponds to no contraction and smaller $\varepsilon$ means the environment direction $u_E$ receives less relative scale.
\end{lemma}

\key{Uniform confusion is thus a \emph{targeted} contraction:} only the directions the adversary uses to separate environments are compressed, more strongly the larger their singular value $s_j$, while directions the adversary ignores, which may carry stable task-relevant information, are left intact. This selectivity is what distinguishes confusion from naive information removal: it narrows the environment-discriminative subspace without collateral damage to invariant features. The rank-one model~\eqref{eq:contraction_model} is a first-order consequence of a dominant singular value, and it ties $\varepsilon$ to the training hyperparameters: after $T$ encoder steps, $\varepsilon \approx \prod_{t=1}^{T}(1-\alpha\beta s_1^2/K)$, so a larger confusion weight $\beta$ or longer training yields a stronger contraction. A neural encoder may additionally rotate or nonlinearly deform the representation, so we verify the model empirically rather than assume it globally.

\begin{proposition}[Ridge identity]
\label[proposition]{prop:ridge}
For any loss $\ell$, any $\lambda > 0$, and any $\varepsilon \in (0,1]$, the $\ell_2$-regularized linear probe on the contracted representation is equivalent to an anisotropic ridge penalty on the original representation:
\begin{equation}
\min_{w,b}\; \mathbb{E}\!\left[ \ell(Y, w^\top T_\varepsilon h + b) \right] + \lambda \|w\|_2^2
\;=\;
\min_{v,b}\; \mathbb{E}\!\left[ \ell(Y, v^\top h + b) \right] + \lambda \!\left[ \|v\|_2^2 + (\varepsilon^{-2}\!-\!1)(v^\top u_E)^2 \right],
\label{eq:ridge_identity}
\end{equation}
where optima correspond under $v = T_\varepsilon w$. The effective penalty on $\operatorname{span}(u_E)$ is $\lambda / \varepsilon^2$, while orthogonal directions retain penalty $\lambda$.
\end{proposition}

\begin{corebox}
\key{Contracting $u_E$ in the representation is \emph{exactly} equivalent to increasing the downstream ridge penalty in that direction from $\lambda$ to $\lambda/\varepsilon^2$, while leaving other directions unaffected.}
\end{corebox}
 The added term $(\varepsilon^{-2}-1)(v^\top u_E)^2$ charges only the component of $v$ along $u_E$ and leaves any $v \perp u_E$ untouched. This is \textbf{feature reweighting, not feature erasure}: the probe can still use the environment-associated direction, but doing so becomes progressively more expensive as $\varepsilon$ decreases, and the information along $u_E$ is never removed. \pcref{prop:ridge} thus performs the reweighting of DFR~\citep{kirichenko2023last} in representation space, as an anisotropic ridge penalty, and therefore needs neither class labels nor group annotations. \key{Regularization is essential:} $T_\varepsilon$ is invertible for $\varepsilon>0$, so an unregularized classifier simply rescales its weight along $u_E$ by $1/\varepsilon$ and contraction has no downstream effect.

\paragraph{When does contraction help?} A two-dimensional Gaussian model (\pcref{prop:worked_example}, \textcolor{paperaccent}{Appendix~\ref{app:worked_example}}) makes the answer exact: when the environment feature is spuriously correlated with $Y$ in training but not at test time, the OOD risk is \emph{strictly decreasing} in the contraction amount; as $\lambda \to 0^+$ the probe undoes the contraction and the effect vanishes; and when $u_E = u_Y$ contraction hurts because the compressed direction carries the stable signal. \key{The mechanism therefore requires exactly two things:} \pointmark{172}~a spurious training correlation along a direction separated from the core signal, and \pointmark{173}~a finitely regularized probe.

%% file: Sections/experiments.tex
\section{Experiments}
\label{sec:experiments}

\subsection{Synthetic Evaluation}
\label{sec:synthetic}

We isolate the mechanism of \pcref{sec:theory} by directly contracting an environment-associated direction in a synthetic representation and checking downstream behavior against the theory.

\begin{figure}[t]
\centering
\includegraphics[width=0.76\linewidth]{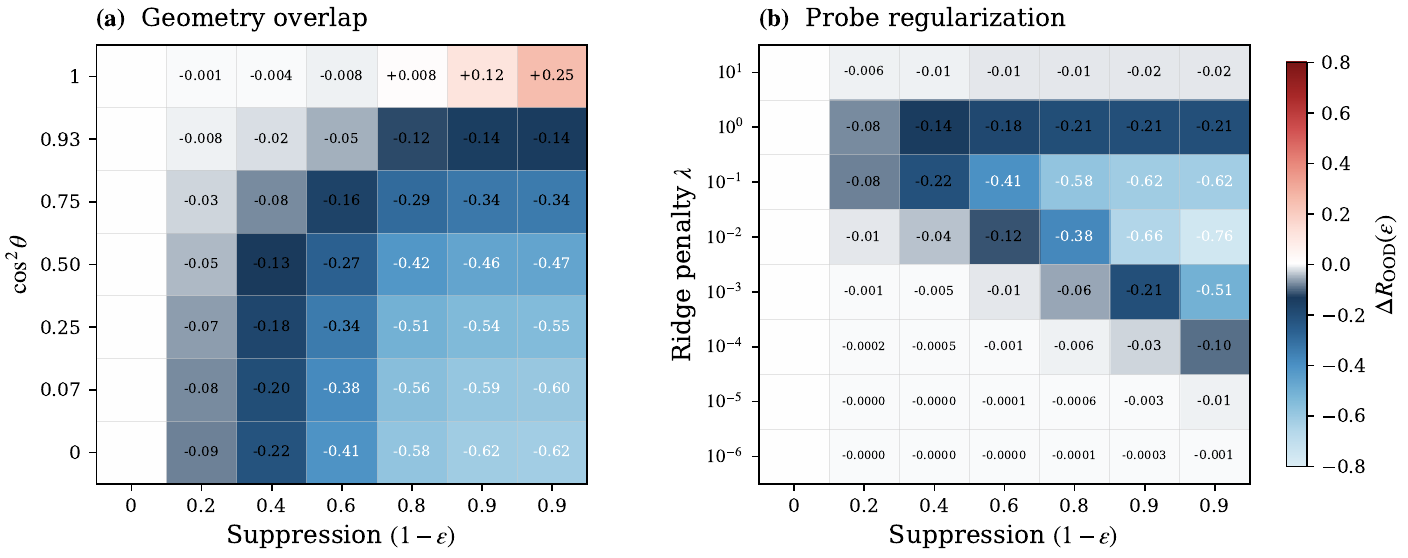}
\caption{\textbf{Controlled contraction of an environment-associated direction.} Cells report $\Delta R_{\mathrm{OOD}}$ (10 seeds; negative = improvement). \textbf{(a)}~Contraction helps when environment and task directions are separated but hurts at full alignment. \textbf{(b)}~The effect vanishes as $\lambda \to 0$ and peaks at intermediate $\lambda$, consistent with the induced penalty $\lambda/\varepsilon^2$.}
\label{fig:synthetic_heatmaps}
\end{figure}

\paragraph{Setup.}
We sample $z_c, z_e \sim \mathcal{N}(0,1)$, $\eta \sim \mathcal{N}(0, \sigma_\eta^2)$, and set $Y = z_c$. The representation is $\mbf{h} = (z_c + \eta)\mbf{u}_Y + \sigma_s(r z_c + \sqrt{1-r^2}\, z_e)\mbf{u}_E$, where $\mbf{u}_E$ and $\mbf{u}_Y$ have squared overlap $\cos^2\theta$. Training uses $r_{\mathrm{tr}} = 0.9$; testing uses $r_{\mathrm{te}} = 0$. We apply $T_\varepsilon = I - (1-\varepsilon)\mbf{u}_E\mbf{u}_E^\top$ directly and fit a ridge probe on $10\text{k}$ training examples, evaluating OOD MSE on $50\text{k}$ test examples. We report $\Delta R_{\mathrm{OOD}} = R_{\mathrm{OOD}}(\varepsilon) - R_{\mathrm{OOD}}(1)$.

\paragraph{\obs{Finding 1:} Geometry governs the cost of contraction.}
We fix $\lambda = 0.1$ and sweep $\cos^2\theta$. When the two directions are orthogonal, contraction reduces OOD error; the benefit shrinks with increasing overlap; and at full alignment contraction \emph{increases} error because it suppresses stable signal alongside the shortcut (\pcref{fig:synthetic_heatmaps}a), matching the trade-off in \pcref{prop:worked_example}(iii).

\paragraph{\obs{Finding 2:} Finite regularization is necessary.}
At near-zero regularization ($\lambda = 10^{-6}$), contraction has negligible effect: the probe rescales its coefficient to undo the contraction, as \pcref{prop:worked_example}(ii) predicts. The benefit peaks at intermediate $\lambda$ and diminishes under excessive regularization, which suppresses all directions uniformly (\pcref{fig:synthetic_heatmaps}b). \key{Contraction acts through the regularization geometry, not by erasing information.}

\subsection{Controlled Graph Benchmark}
\label{sec:goodmotif}

The synthetic study applies $T_\varepsilon$ to a fixed representation; we now test whether the mechanism survives end-to-end training on realistic graphs.

\paragraph{Setup.}
We use GOOD-Motif concept shift~\citep{gui2022good}, where the base-graph distribution shifts across domains while the motif--label relationship is stable. The benchmark provides a domain index for every training graph, which we use directly as the environment $e$ ($K$ equals the number of training domains); class labels are never seen by the encoder. The encoder is a 3-layer GCN with a 64-dimensional latent space, trained with the VGAE reconstruction loss and fine-tuned with adversarial confusion. We freeze the encoder and fit an $\ell_2$-regularized linear classifier on the training labels, sweeping the inverse regularization strength over $C \in \{10^{-3}, 10^{-2}, \dots, 10^{3}\}$. To isolate the effect of representation geometry from trivial scale differences, we fit a per-dimension standardizer on the $\beta{=}0$ training representations and apply it, frozen, to all encoder variants before probing. We report OOD test accuracy averaged over 3 seeds.

\paragraph{Results.}
\pcref{tab:goodmotif} reports OOD test accuracy. \key{\ours outperforms GraphCL, MVGRL, and OOD-GCL, and improves over the plain VGAE warm-start.} RGCL performs on par with \ours on this setting.

\paragraph{Ablation: probe regularization.}
\pcref{fig:reg_sweep} sweeps $C$ over six orders of magnitude. Accuracy for both $\beta{=}0$ and $\beta{=}0.1$ is low when the probe is too constrained to use any feature, peaks near $C{=}10$, and drops once weak regularization lets the probe exploit shortcut features. The gain $\Delta$ from confusion is non-monotonic: it is largest ($+3.8\%$) at $C{=}10^{-3}$, where a large base penalty $\lambda$ is inflated to $\lambda/\varepsilon^2$ on the environment direction (\pcref{prop:ridge}); it vanishes at $C{=}1$, where the probe has enough capacity to partially undo the contraction (\pcref{prop:worked_example}(ii)); and it returns at $C \ge 10^2$ ($+1.4$, $+2.4\%$), where the baseline probe freely exploits shortcuts that the confusion-trained encoder has already contracted.

\begin{figure}[t]
\centering
\begin{minipage}[c]{0.46\linewidth}
\centering
\captionof{table}{\textbf{GOOD-Motif concept shift}, OOD test accuracy (\%). \textbf{Bold}: best; \underline{underline}: second best.}
\label{tab:goodmotif}
\small
\begin{tabular}{@{}lc@{}}
\toprule
\rowcolor{tabhead} Method & Acc. \\
\midrule
ERM (sup.) & 53.0 {\scriptsize$\pm$8.5} \\
GraphCL & 60.9 {\scriptsize$\pm$3.4} \\
MVGRL & 54.5 {\scriptsize$\pm$4.7} \\
OOD-GCL & 64.1 {\scriptsize$\pm$3.6} \\
RGCL & \textbf{70.1} {\scriptsize$\pm$6.5} \\
\midrule
VGAE ($\beta{=}0$) & 65.9 {\scriptsize$\pm$2.0} \\
\rowcolor{tabours} \ours & \underline{69.6} {\scriptsize$\pm$2.2} \\
\bottomrule
\end{tabular}
\end{minipage}%
\hfill
\begin{minipage}[c]{0.36\linewidth}
\centering
\includegraphics[width=\linewidth]{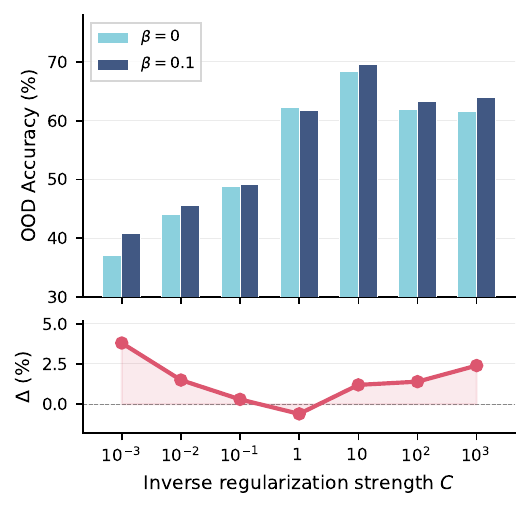}
\caption{\textbf{Probe regularization sweep on GOOD-Motif.} \textit{Top}: OOD accuracy for $\beta{=}0$ vs.\ $\beta{=}0.1$. \textit{Bottom}: gain $\Delta$ from confusion.}
\label{fig:reg_sweep}
\end{minipage}
\end{figure}

\input{tables/realworld_drugood_results}
\subsection{Real-World Molecular Evaluation}
\label{sec:real_world}

\paragraph{Setup.}
We use EC50 and IC50 from DrugOOD~\citep{ji2022drugood} under assay, scaffold, and size splits. Each encoder is frozen and a logistic-regression probe is fitted on standardized representations using 5\%, 10\%, or 20\% of the labeled training split; $(\alpha,\beta,K)$ and checkpoints are selected on validation and the test split is used once. Baselines are GraphCL, RGCL, OOD-GCL, MVGRL, and VGAE ($\beta{=}0$), i.e.\ \ours without confusion. Details are in \textcolor{paperaccent}{Appendix~\ref{app:realworld_details}}.

\paragraph{Environment partition.}
Unlike GOOD-Motif, DrugOOD offers no usable environment index inside the training split: the assay, scaffold, and size attributes define the train/validation/test split, but within training they are respectively thousands of sparsely populated assay IDs, near-unique scaffolds, and a continuous graph size, none of which yields a partition of the kind \pcref{sec:problem_setting} assumes. We therefore construct \emph{pseudo-environments} without labels, following~\citet{yao2024empowering}: a frozen MVGRL encoder embeds the training graphs, $k$-means with $K$ clusters is run on the embeddings, and the cluster index serves as $e_i$. Cluster imbalance is absorbed by the rebalancing exponent $\alpha$ of \pcref{sec:adv_training}. No annotation of any kind is used during representation learning, and $K$ is treated as a hyperparameter selected on validation.

\paragraph{Results.}
\pcref{tab:drugood_main} reports test ROC-AUC. \key{\ours is best in 17 of the 18 dataset--ratio cells,} the sole exception being IC50-size at 5\% labels, and the margin is largest on the scaffold splits, where test molecules differ structurally from training ones. The VGAE ($\beta{=}0$) row is the ablation without confusion: it is the weakest representation on most splits, trailing \ours by 11--15 points on EC50-scaffold and by 3--10 points elsewhere, \key{so the gain is attributable to confusion training rather than to the reconstruction backbone.} MVGRL, whose embedding only supplies the pseudo-environment partition, is likewise well below \ours, so a rough, label-free partition is enough for confusion to act on. Averaged over the six splits, \ours leads the strongest baseline by $4.1$, $3.7$, and $2.7$ points at 5\%, 10\%, and 20\% labels, so the gain holds at every label budget.

\paragraph{Sensitivity.}
Across the top-ranked configurations (\textcolor{paperaccent}{Appendix~\ref{app:realworld_hparams}}), the strongest confusion weight $\beta=10$ is preferred in nearly every setting, whereas the best $K$ varies across splits and is selected on validation. The preference for large $\beta$ matches \pcref{prop:ridge}: a probe fitted on 5--20\% of the labels sits in the strongly regularized regime where \pcref{fig:synthetic_heatmaps}(b) shows contraction to help most. Component-loss curves (\textcolor{paperaccent}{Appendix~\ref{app:realworld_losses}}) confirm that the alternating optimization is stable.

%% file: tables/realworld_drugood_results.tex
\begin{table*}[t]
\centering
\caption{\textbf{DrugOOD under limited labels.} Test ROC-AUC (\%), mean $\pm$ std over three seeds (IC50-size VGAE at 20\%: two), for linear probes on frozen representations. \textbf{Bold}: best per column; Avg.: mean over the six splits; shaded row: \ours.}
\label{tab:drugood_main}
\setlength{\tabcolsep}{4.2pt}
\renewcommand{\arraystretch}{0.92}
\scalebox{0.86}{%
\begin{tabular}{lccccccc}
\toprule
\rowcolor{tabhead} Method & \multicolumn{3}{c}{EC50} & \multicolumn{3}{c}{IC50} & \\
\cmidrule(lr){2-4}\cmidrule(lr){5-7}
\rowcolor{tabhead} & Assay & Scaffold & Size & Assay & Scaffold & Size & Avg. \\
\midrule
\multicolumn{8}{l}{\textcolor{paperobs}{\textit{5\% labels}}} \\
MVGRL & 68.9$\pm$3.7 & 52.1$\pm$3.7 & 57.8$\pm$2.9 & 67.1$\pm$0.3 & 57.4$\pm$1.7 & \textbf{62.1$\pm$0.8} & 60.9 \\
VGAE ($\beta{=}0$) & 58.2$\pm$2.8 & 50.0$\pm$6.2 & 52.9$\pm$4.1 & 61.5$\pm$2.2 & 55.7$\pm$2.2 & 56.8$\pm$0.7 & 55.9 \\
GraphCL & 65.1$\pm$2.0 & 54.6$\pm$9.0 & 56.3$\pm$1.8 & 67.0$\pm$1.0 & 61.8$\pm$1.9 & 61.5$\pm$1.1 & 61.1 \\
RGCL & 69.5$\pm$3.1 & 54.8$\pm$7.5 & 56.5$\pm$2.9 & 65.8$\pm$0.1 & 60.0$\pm$1.6 & 60.5$\pm$1.1 & 61.2 \\
OOD-GCL & 59.5$\pm$6.5 & 57.3$\pm$6.2 & 53.4$\pm$1.1 & 58.1$\pm$2.2 & 57.3$\pm$4.0 & 57.1$\pm$2.9 & 57.1 \\
\rowcolor{tabours} \ours & \textbf{70.6$\pm$1.0} & \textbf{65.0$\pm$1.0} & \textbf{61.0$\pm$1.8} & \textbf{69.5$\pm$0.5} & \textbf{64.3$\pm$1.2} & 61.6$\pm$1.0 & \textbf{65.3} \\
\midrule
\multicolumn{8}{l}{\textcolor{paperobs}{\textit{10\% labels}}} \\
MVGRL & 67.5$\pm$2.5 & 54.4$\pm$3.5 & 56.0$\pm$1.8 & 67.8$\pm$0.2 & 59.8$\pm$1.0 & 61.7$\pm$1.1 & 61.2 \\
VGAE ($\beta{=}0$) & 59.6$\pm$2.0 & 49.9$\pm$1.9 & 53.2$\pm$1.9 & 62.1$\pm$1.0 & 58.4$\pm$1.6 & 58.6$\pm$0.3 & 57.0 \\
GraphCL & 66.9$\pm$1.7 & 58.8$\pm$2.5 & 58.4$\pm$0.7 & 67.3$\pm$0.7 & 62.8$\pm$0.8 & 61.5$\pm$0.4 & 62.6 \\
RGCL & 69.8$\pm$1.5 & 58.8$\pm$2.3 & 57.7$\pm$2.6 & 66.9$\pm$1.3 & 60.8$\pm$1.7 & 61.3$\pm$0.6 & 62.6 \\
OOD-GCL & 63.9$\pm$4.7 & 57.2$\pm$3.1 & 53.2$\pm$1.0 & 59.0$\pm$1.6 & 58.8$\pm$4.0 & 58.3$\pm$2.9 & 58.4 \\
\rowcolor{tabours} \ours & \textbf{71.8$\pm$1.7} & \textbf{64.9$\pm$1.2} & \textbf{62.9$\pm$1.5} & \textbf{69.5$\pm$0.1} & \textbf{66.0$\pm$0.7} & \textbf{62.6$\pm$0.9} & \textbf{66.3} \\
\midrule
\multicolumn{8}{l}{\textcolor{paperobs}{\textit{20\% labels}}} \\
MVGRL & 67.8$\pm$1.6 & 56.3$\pm$1.2 & 57.1$\pm$2.4 & 67.8$\pm$0.5 & 61.4$\pm$0.8 & 61.6$\pm$0.3 & 62.0 \\
VGAE ($\beta{=}0$) & 62.5$\pm$4.6 & 53.1$\pm$0.6 & 55.6$\pm$1.4 & 63.9$\pm$0.4 & 60.4$\pm$0.7 & 60.2$\pm$0.0 & 59.3 \\
GraphCL & 70.7$\pm$3.7 & 59.1$\pm$3.7 & 60.2$\pm$0.4 & 67.5$\pm$0.3 & 63.5$\pm$0.7 & 62.3$\pm$0.7 & 63.9 \\
RGCL & 73.0$\pm$6.2 & 55.0$\pm$4.3 & 61.0$\pm$0.6 & 67.2$\pm$0.4 & 61.3$\pm$1.1 & 62.3$\pm$1.0 & 63.3 \\
OOD-GCL & 64.3$\pm$4.0 & 58.7$\pm$4.3 & 53.9$\pm$1.0 & 60.1$\pm$1.9 & 58.7$\pm$3.6 & 59.0$\pm$3.1 & 59.1 \\
\rowcolor{tabours} \ours & \textbf{73.6$\pm$1.3} & \textbf{64.4$\pm$1.7} & \textbf{62.8$\pm$0.5} & \textbf{69.7$\pm$0.1} & \textbf{65.9$\pm$0.4} & \textbf{63.2$\pm$0.5} & \textbf{66.6} \\
\bottomrule
\end{tabular}
}
\end{table*}

%% file: Sections/conclusion.tex
\section{Conclusion}
\label{sec:conclusion}

We introduced \ours, a label-free framework for graph OOD generalization that pairs a VGAE with adversarial confusion against the training environments, given or, as on DrugOOD, inferred by clustering. Both components come with a guarantee: reconstruction forces the latent to retain core and environment factors alike, and confusion contracts, rather than erases, the environment-discriminative directions, which for any $\ell_2$-regularized probe is exactly an anisotropic ridge penalty. Experiments on a controlled synthetic model, GOOD-Motif, and DrugOOD confirm both the mechanism and its preconditions. The recipe, i.e., \key{retain by reconstruction, then reweight by contraction}, is not specific to graphs and extends naturally to other likelihood-based encoders.

\section*{Ethics Statement}
This work studies unsupervised representation learning for graph out-of-distribution generalization. All experiments use publicly available benchmarks that contain no personal data or human subjects. By removing the need for class labels during representation learning and tolerating coarse or inferred environment partitions, \ours lowers the cost of building predictors that remain reliable under distribution shift, which is beneficial in settings such as molecular property prediction where labeled and annotated data are scarce. The method does not create new data-collection or surveillance capabilities, and we do not foresee any direct negative societal impact beyond those generic to representation learning.

%% file: Sections/appendix.tex
\section{Proofs}
\label{app:proofs}

\subsection{Proof of Theorem~\ref{thm:retention} (Reconstruction Forces Retention)}
\label{app:retention_proof}

Throughout, $G = g(C, S, \xi)$ is a discrete random variable, $Z = f_\theta(G)$ is the (possibly stochastic) encoder latent, so that $(C, S, \xi) \to G \to Z$ forms a Markov chain, and $p_\psi(\cdot \mid Z)$ is the decoder distribution over graphs. All entropies are in nats.

\textbf{Step 1: Reconstruction loss upper-bounds the conditional entropy of $G$.}
For every value $z$ of $Z$, the Gibbs inequality gives $\mathbb{E}\!\left[-\log p_\psi(G \mid z) \,\middle|\, Z = z\right] \ge H(G \mid Z = z)$, with equality iff $p_\psi(\cdot \mid z)$ equals the true posterior $p(\cdot \mid z)$. Taking expectation over $Z$,
\begin{equation}
L_{\mathrm{rec}} = \mathbb{E}\!\left[-\log p_\psi(G \mid Z)\right] \;\ge\; H(G \mid Z).
\label{eq:rec_ge_condent}
\end{equation}
This is the only place the decoder enters: any encoder--decoder pair achieving reconstruction loss $L_{\mathrm{rec}}$ leaves at most $L_{\mathrm{rec}}$ nats of $G$ unexplained by $Z$.

\textbf{Step 2: Injectivity turns $S$ into a function of $(G, C, \xi)$.}
By (A1), for each fixed $(c, \xi')$ the map $s \mapsto g(c, s, \xi')$ is injective, so there exists a measurable $\phi_S$ with $S = \phi_S(G, C, \xi)$ almost surely. Symmetrically, injectivity in the first argument yields $\phi_C$ with $C = \phi_C(G, S, \xi)$.

\textbf{Step 3: Chain of conditioning inequalities.}
Since $S$ is a deterministic function of $(G, C, \xi)$, conditioning on $(Z, C, \xi)$ and using that the entropy of a function is at most the entropy of its argument,
\begin{equation}
H(S \mid Z, C, \xi) \;=\; H\!\left(\phi_S(G, C, \xi) \mid Z, C, \xi\right) \;\le\; H(G \mid Z, C, \xi) \;\le\; H(G \mid Z),
\label{eq:cond_chain}
\end{equation}
where the last step is ``conditioning reduces entropy.'' Combining with~\eqref{eq:rec_ge_condent}, $H(S \mid Z, C, \xi) \le L_{\mathrm{rec}}$.

\textbf{Step 4: Conclusion.}
By the chain rule for conditional mutual information,
\begin{equation}
I(Z; S \mid C, \xi) \;=\; H(S \mid C, \xi) - H(S \mid Z, C, \xi) \;\ge\; H(S \mid C, \xi) - L_{\mathrm{rec}},
\end{equation}
which is the first bound in~\eqref{eq:retention_bounds}. The second bound follows by exchanging the roles of $C$ and $S$ in Steps~2--3 using $\phi_C$. Both bounds are strictly positive whenever $L_{\mathrm{rec}} < \min\{H(S \mid C, \xi),\, H(C \mid S, \xi)\}$, and (A2) guarantees the right-hand sides are finite so the statement is non-vacuous. \qed

\textbf{Why the bounds are per-factor.}
The two inequalities hold \emph{simultaneously} for the same $L_{\mathrm{rec}}$ because each factor is recoverable from $G$ once the others are fixed; the reconstruction loss therefore has to pay for every factor separately. A discriminative objective offers no such guarantee: minimizing $\mathbb{E}[-\log q(Y \mid Z)]$ only constrains $I(Z; Y) \le H(Y)$, and an encoder that discards $C$ can still attain this whenever $S$ predicts $Y$ on the training environments.

\subsection{Proof of Lemma~\ref{lem:contraction} (Local Contraction)}

We first establish the local contraction result in four steps (Steps 1--4): deriving the confusion gradient, introducing centering, linearizing near uniform predictions, and showing per-direction contraction. We then derive the rank-one cumulative approximation in Steps 5--8.

\textbf{Step 1: Confusion gradient in representation space.}
The per-example confusion loss is cross-entropy with the uniform target $q = K^{-1}\mathbf{1}$:
\begin{equation}
L_{\mathrm{conf}}(h) = -q^\top \log p(h), \qquad p(h) = \operatorname{softmax}(Wh + b).
\end{equation}
The standard logit-space gradient of cross-entropy with target $q$ is $\nabla_z L_{\mathrm{conf}} = p - q$. Using $z = Wh + b$, the chain rule gives:
\begin{equation}
\begin{aligned}
\nabla_h L_{\mathrm{conf}}
&= \frac{\partial z}{\partial h}^\top \nabla_z L_{\mathrm{conf}} \\
&= W^\top(p - q).
\end{aligned}
\end{equation}
That is, the representation gradient is determined entirely by the directions the adversary uses to distinguish environments.

\textbf{Step 2: Centering removes the common logit shift.}
Since both $p$ and $q$ are probability vectors summing to one, their difference satisfies $\mathbf{1}^\top(p-q) = 0$, so $p - q \in \mathbf{1}^\perp$ and $C(p-q) = p - q$ where $C = I - K^{-1}\mathbf{1}\mathbf{1}^\top$. Therefore:
\begin{equation}
\begin{aligned}
\nabla_h L_{\mathrm{conf}}
&= W^\top(p - q) \\
&= W^\top C(p - q) \\
&= (CW)^\top(p - q) \\
&= W_c^\top(p - q),
\end{aligned}
\end{equation}
which shows the gradient lies in $\operatorname{range}(W_c^\top)$. Since $C$ has rank $K-1$, this subspace has dimension at most $K-1$. Centering removes the common logit shift and keeps only the environment-discriminative directions.

\textbf{Step 3: Linearization near equal-logit points.}
At $h_0$ where $W_c h_0 + b_c = 0$, all centered logits vanish, hence all logits are equal and $p(h_0) = q$. The Jacobian of softmax evaluated at equal probabilities is:
\begin{equation}
\begin{aligned}
J_{\mathrm{sm}}(z(h_0))
&= \operatorname{diag}(q) - qq^\top \\
&= \frac{1}{K}\!\left(I - \frac{1}{K}\mathbf{1}\mathbf{1}^\top\right) \\
&= \frac{1}{K}C.
\end{aligned}
\end{equation}
A first-order Taylor expansion of $p$ around $h_0$, using $\delta h = h - h_0$ and noting that $CW\delta h = W_c \delta h$, gives:
\begin{equation}
p(h) - q = \frac{1}{K} W_c \delta h + O(\|W_c \delta h\|_2^2).
\end{equation}
Substituting into the gradient expression from Step~2:
\begin{equation}
\begin{aligned}
\nabla_h L_{\mathrm{conf}}
&= W_c^\top(p - q) \\
&= W_c^\top\!\left(\frac{1}{K} W_c \delta h + O(\|W_c \delta h\|_2^2)\right) \\
&= \frac{1}{K} W_c^\top W_c\, \delta h + O(\|W_c \delta h\|_2^2).
\end{aligned}
\end{equation}

\textbf{Step 4: Contraction per singular direction.}
A gradient step $h^+ = h - \alpha\beta\,\nabla_h L_{\mathrm{conf}}$ gives, in terms of $\delta h$:
\begin{equation}
\delta h^+ = \left(I - \frac{\alpha\beta}{K} W_c^\top W_c\right) \delta h + O(\alpha\beta \|W_c \delta h\|_2^2).
\end{equation}
Let $W_c = U \operatorname{diag}(s_1, \ldots, s_r) V^\top$ be a thin SVD. Then:
\begin{equation}
W_c^\top W_c = V \operatorname{diag}(s_1^2, \ldots, s_r^2) V^\top,
\end{equation}
and projecting onto each right singular direction $v_j$:
\begin{equation}
v_j^\top \delta h^+ = \left(1 - \frac{\alpha\beta}{K} s_j^2\right) v_j^\top \delta h + O(\alpha\beta \|W_c \delta h\|_2^2).
\end{equation}
When $0 < \alpha\beta s_j^2 / K < 1$, the leading-order factor is a contraction. Directions with larger adversarial singular values are contracted more strongly, meaning the representation selectively suppresses environment-discriminative components. For any direction $v \perp \operatorname{range}(V) = \operatorname{range}(W_c^\top)$, we have $W_c^\top W_c\, v = 0$, so $v^\top \delta h^+ = v^\top \delta h$ exactly; non-discriminative directions are unchanged.

\textbf{Step 5: Rank-one approximation.}
When the adversary's centered weight matrix has a dominant singular value $s_1 \gg s_2$, the leading-order update from Step~4 is well approximated by retaining only the $j=1$ term. Writing $u_E := v_1$ for the leading right singular direction and defining $\gamma := \alpha\beta s_1^2 / K$, the single-step update becomes:
\begin{equation}
h^{(1)} \approx (I - \gamma\, u_E u_E^\top)\, h^{(0)}.
\end{equation}

\textbf{Step 6: Decomposition into environment and orthogonal components.}
To see how the update acts on different components, decompose any representation as $h = h_E + h_\perp$, where $h_E = (u_E^\top h)\, u_E$ is the environment-direction component and $u_E^\top h_\perp = 0$. Denoting the projection $P = u_E u_E^\top$ (so $P^2 = P$), one confusion step gives:
\begin{equation}
\begin{aligned}
h^{(1)}
&= (I - \gamma P)(h_E + h_\perp) \\
&= h_E + h_\perp - \gamma P\, h_E \\
&= (1 - \gamma)\, h_E + h_\perp.
\end{aligned}
\end{equation}
That is, the environment component shrinks by a factor $(1 - \gamma)$, while the orthogonal component is unchanged.

\textbf{Step 7: Cumulative effect of $m$ confusion steps.}
Repeating the same rank-one update $m$ times and using $P^2 = P$:
\begin{equation}
\begin{aligned}
h^{(m)}
&= (I - \gamma P)^m\, h^{(0)} \\
&= (1 - \gamma)^m\, h_E^{(0)} + h_\perp^{(0)}.
\end{aligned}
\end{equation}
Each step multiplicatively shrinks the environment component while leaving the orthogonal complement intact. Defining $\varepsilon := (1-\gamma)^m \in (0, 1]$, we can write:
\begin{equation}
\begin{aligned}
h^{(m)}
&= \varepsilon\, h_E^{(0)} + h_\perp^{(0)} \\
&= h^{(0)} - (1 - \varepsilon)(u_E^\top h^{(0)})\, u_E.
\end{aligned}
\end{equation}

\textbf{Step 8: The contraction model $T_\varepsilon$.}
The cumulative multi-step effect can therefore be expressed as a single linear map:
\begin{equation}
\mathbf{h}_G^{(\beta)} \approx T_\varepsilon\, \mathbf{h}_G^{(0)}, \qquad T_\varepsilon = I - (1-\varepsilon)\, u_E u_E^\top, \quad 0 < \varepsilon \le 1.
\end{equation}
This map has a transparent geometric action:
\begin{equation}
\begin{aligned}
u_E^\top T_\varepsilon h &= \varepsilon\, (u_E^\top h), \\
v^\top T_\varepsilon h &= v^\top h \quad \text{for } v \perp u_E.
\end{aligned}
\end{equation}
When $\varepsilon = 1$, $T_\varepsilon = I$ (no contraction); as $\varepsilon \to 0$, the environment-discriminative component is increasingly suppressed, while all other directions are preserved. \hfill$\blacksquare$

\textbf{Scope.} Steps~1--4 are exact for the gradient direction and first-order for the contraction magnitude, holding for a single step with a fixed linear adversary and a directly updated representation. Steps~5--8 model the cumulative multi-step effect as a rank-one contraction under three simplifying assumptions: (a)~the adversary weights (hence $u_E$) remain approximately stable across steps; (b)~the encoder update acts directly on $h$ rather than through neural network parameters; and (c)~the dominant singular value $s_1$ governs the contraction. In the implemented method, the encoder is updated through its parameter Jacobian (introducing a preconditioner), the adversary is nonlinear and alternately retrained, and the reconstruction loss acts simultaneously. The rank-one model $T_\varepsilon$ is therefore a first-order analytical approximation that captures the leading effect of confusion training, suppressing the dominant environment-discriminative direction while preserving orthogonal components, and is verified empirically rather than derived as a theorem.

\subsection{Proof of Proposition~\ref{prop:ridge} (Ridge Identity)}

The key insight is that contraction $T_\varepsilon$ does not remove information; rather, it changes the downstream regularization geometry. We establish this through a change of variables.

\textbf{Step 1: Invertibility of $T_\varepsilon$.} For $\varepsilon > 0$, $T_\varepsilon$ is symmetric and invertible with:
\begin{equation}
T_\varepsilon^{-1} = I + (\varepsilon^{-1} - 1)\, u_E u_E^\top.
\end{equation}
To verify, using $u_E^\top u_E = 1$:
\begin{equation}
\begin{aligned}
T_\varepsilon \cdot T_\varepsilon^{-1}
&= (I - (1-\varepsilon)\, u_E u_E^\top)(I + (\varepsilon^{-1}-1)\, u_E u_E^\top) \\
&= I + \bigl[(\varepsilon^{-1}-1) - (1-\varepsilon) - (1-\varepsilon)(\varepsilon^{-1}-1)\bigr] u_E u_E^\top \\
&= I,
\end{aligned}
\end{equation}
since the bracketed coefficient simplifies to $\varepsilon^{-1} - 1 - 1 + \varepsilon - \varepsilon^{-1} + 1 + 1 - \varepsilon = 0$.

\textbf{Step 2: Squared inverse.} By squaring $T_\varepsilon^{-1}$ using the identity $(I + \gamma\, u_E u_E^\top)^2 = I + ((1+\gamma)^2 - 1)\, u_E u_E^\top$ with $\gamma = \varepsilon^{-1} - 1$:
\begin{equation}
T_\varepsilon^{-2} = I + (\varepsilon^{-2} - 1)\, u_E u_E^\top.
\end{equation}

\textbf{Step 3: Change of variables.} Substitute $v := T_\varepsilon w$, which is a bijection by Step~1, leaving $b$ unchanged. The prediction function is unchanged after reparameterization:
\begin{equation}
\begin{aligned}
w^\top T_\varepsilon h
&= (T_\varepsilon w)^\top h \\
&= v^\top h,
\end{aligned}
\end{equation}
so the loss terms agree pointwise.

\textbf{Step 4: Transform the regularizer.} For the penalty term, $w = T_\varepsilon^{-1} v$, and only the environment direction receives an additional penalty:
\begin{equation}
\begin{aligned}
\|w\|_2^2
&= \|T_\varepsilon^{-1} v\|_2^2 \\
&= v^\top T_\varepsilon^{-2} v \\
&= v^\top\!\left[I + (\varepsilon^{-2}-1)\, u_E u_E^\top\right] v \\
&= \|v\|_2^2 + (\varepsilon^{-2}-1)(v^\top u_E)^2.
\end{aligned}
\end{equation}
The two objectives are therefore equal as functions under the bijection $v = T_\varepsilon w$, so their infima agree and optima correspond whenever they exist. The penalty matrix $\lambda T_\varepsilon^{-2}$ has eigenvalue $\lambda/\varepsilon^2$ on $\operatorname{span}(u_E)$ and $\lambda$ on $u_E^\perp$. Therefore, contraction is equivalent to directional feature reweighting: using the environment direction becomes disproportionately expensive under $\ell_2$ regularization. \hfill$\blacksquare$

\subsection{Exact Worked Example: Statement and Proof}
\label{app:worked_example}

The mechanism improves OOD generalization when the contracted direction captures unstable environment-associated variation while remaining separated from invariant task-relevant information. We formalize this intuition using a two-dimensional Gaussian model; the statement is the one summarized in Section~\ref{sec:theory}. The analysis highlights two essential requirements: (1)~the environment-associated feature induces a spurious training correlation, and (2)~the downstream classifier applies finite regularization.

\begin{proposition}[Exact worked example]
\label[proposition]{prop:worked_example}
Let $h = (z_c + \eta)\, u_Y + \sigma_s (r\, z_c + \sqrt{1-r^2}\, z_e)\, u_E$ with $Y = z_c$, orthonormal $(u_Y, u_E)$, $z_c, z_e \sim \mathcal{N}(0,1)$ and $\eta \sim \mathcal{N}(0, \sigma_\eta^2)$ mutually independent, $\sigma_\eta^2 > 0$, $\sigma_s > 0$, and training correlation $r = r_{\mathrm{tr}}$. Under population ridge regression with regularization $\lambda > 0$:
\begin{enumerate}[label=(\roman*),leftmargin=*]
\item If $r_{\mathrm{tr}} \neq 0$ and $r_{\mathrm{te}} = 0$ (correlation shift), the OOD risk $R(\varepsilon; r_{\mathrm{te}})$ is \emph{strictly decreasing} in the contraction amount: every reduction of $\varepsilon$ helps.
\item As $\lambda \to 0^+$ (no regularization), both fitted coefficients become independent of $\varepsilon$: contraction has no downstream effect.
\item If $u_E = u_Y$ (full alignment), contraction hurts because the compressed direction also contains task-relevant information.
\end{enumerate}
\end{proposition}

This proposition makes the two requirements concrete. Part~(i) shows the positive case: when a spurious correlation exists in training ($r_{\mathrm{tr}} \neq 0$) but vanishes at test time ($r_{\mathrm{te}} = 0$), contraction \emph{strictly} reduces OOD risk. Two effects contribute simultaneously: the probe's loading on the spurious feature decreases, and its core coefficient, which was deflated at $\varepsilon = 1$ because the environment feature could proxy for the label, is restored toward its test-optimal value. Parts~(ii) and~(iii) characterize the failure modes: without regularization the probe can undo the contraction for free, and when the environment direction overlaps the stable signal, contraction damages invariant information. Together, these results precisely delineate when \ours is expected to help and when it is not.

\textbf{Proof.} The proof follows the chain: data model $\to$ covariance $\Sigma$ and cross-covariance $b$ $\to$ ridge solution $\hat{w}_\varepsilon$ $\to$ effective coefficients $a_c, a_s$ $\to$ OOD risk.

\textbf{Step 1: Setup and covariance derivation.} We work in the orthonormal basis $(u_Y, u_E)$, so each representation is $h = (h_1, h_2)$ with $h_1 = z_c + \eta$ (core feature with noise) and $h_2 = \sigma_s(r\, z_c + \sqrt{1-r^2}\, z_e)$ (environment feature). The label is $Y = z_c$. In this basis, the contraction $T_\varepsilon$ acts as $S_\varepsilon = \operatorname{diag}(1, \varepsilon)$, leaving the core feature unchanged and compressing the environment feature.

Since all variables are zero-mean and $z_c, z_e, \eta$ are mutually independent with $\mathbb{E}[z_c^2] = \mathbb{E}[z_e^2] = 1$ and $\mathbb{E}[\eta^2] = \sigma_\eta^2$, the core variance is:
\begin{equation}
\begin{aligned}
\mathbb{E}[h_1^2]
&= \mathbb{E}[(z_c + \eta)^2] \\
&= \mathbb{E}[z_c^2] + \mathbb{E}[\eta^2] \\
&= 1 + \sigma_\eta^2.
\end{aligned}
\end{equation}
The cross-covariance between core and environment features is:
\begin{equation}
\begin{aligned}
\mathbb{E}[h_1 h_2]
&= \sigma_s \mathbb{E}\!\left[(z_c + \eta)(r\, z_c + \sqrt{1-r^2}\, z_e)\right] \\
&= r\sigma_s.
\end{aligned}
\end{equation}
This off-diagonal term directly captures the spurious correlation. Similarly, $\mathbb{E}[h_2^2] = \sigma_s^2$, giving:
\begin{equation}
\Sigma = \mathbb{E}[hh^\top] = \begin{pmatrix} 1 + \sigma_\eta^2 & r\sigma_s \\ r\sigma_s & \sigma_s^2 \end{pmatrix}, \qquad
b = \mathbb{E}[hY] = \begin{pmatrix} 1 \\ r\sigma_s \end{pmatrix}.
\end{equation}

\textbf{Step 2: Ridge solution on contracted features.} The contracted features are $x = S_\varepsilon h$, so $\mathbb{E}[xx^\top] = S_\varepsilon \Sigma S_\varepsilon$ and $\mathbb{E}[xY] = S_\varepsilon b$. The population ridge regression solution is:
\begin{equation}
\begin{aligned}
\hat{w}_\varepsilon
&= (\mathbb{E}[xx^\top] + \lambda I)^{-1} \mathbb{E}[xY] \\
&= (S_\varepsilon \Sigma S_\varepsilon + \lambda I)^{-1} S_\varepsilon b.
\end{aligned}
\end{equation}
The determinant of the $2 \times 2$ system is:
\begin{equation}
\Delta_\varepsilon = (1 + \sigma_\eta^2 + \lambda)(\varepsilon^2 \sigma_s^2 + \lambda) - \varepsilon^2 r^2 \sigma_s^2 > 0 \quad (\lambda > 0).
\end{equation}
The effective coefficients on the original (unshrunk) features, defined by writing the prediction as $\hat{Y} = a_c h_1 + a_s h_2$ with $a_c = [\hat{w}_\varepsilon]_1$ and $a_s = \varepsilon [\hat{w}_\varepsilon]_2$, are:
\begin{equation}
a_c(\varepsilon) = \frac{\lambda + \varepsilon^2 \sigma_s^2(1 - r^2)}{\Delta_\varepsilon}, \qquad
a_s(\varepsilon) = \frac{\varepsilon^2 r \sigma_s(\sigma_\eta^2 + \lambda)}{\Delta_\varepsilon}.
\label{eq:coefficients_app}
\end{equation}

\textbf{Step 3: OOD risk derivation.} At test time, $h_2 = \sigma_s(r_{\mathrm{te}}\, z_c + \sqrt{1-r_{\mathrm{te}}^2}\, z_e)$. The prediction error is:
\begin{equation}
\begin{aligned}
Y - \hat{Y}
&= z_c - a_c(z_c + \eta) - a_s \sigma_s(r_{\mathrm{te}}\, z_c + \sqrt{1-r_{\mathrm{te}}^2}\, z_e) \\
&= (1 - a_c - a_s \sigma_s r_{\mathrm{te}})\, z_c - a_c\, \eta - a_s \sigma_s \sqrt{1-r_{\mathrm{te}}^2}\, z_e.
\end{aligned}
\end{equation}
Since $z_c, \eta, z_e$ are mutually independent and zero-mean with $\mathbb{E}[z_c^2] = \mathbb{E}[z_e^2] = 1$ and $\mathbb{E}[\eta^2] = \sigma_\eta^2$, squaring and taking expectations gives:
\begin{equation}
\begin{aligned}
R(\varepsilon; r_{\mathrm{te}})
&= \mathbb{E}[(Y - \hat{Y})^2] \\
&= (1 - a_c - a_s \sigma_s r_{\mathrm{te}})^2 \\
&\quad + a_c^2 \sigma_\eta^2 \\
&\quad + a_s^2 \sigma_s^2(1 - r_{\mathrm{te}}^2).
\end{aligned}
\label{eq:ood_risk_app}
\end{equation}
This decomposition is transparent: the first term is the squared bias from under- or over-relying on the core feature, the second is the noise penalty from using the noisy core channel, and the third is the variance penalty from loading on the environment feature whose correlation shifts.

We now prove each claim.

\textbf{Claim (i): Strict benefit under correlation shift.}
Let $r = r_{\mathrm{tr}} \neq 0$, $r_{\mathrm{te}} = 0$, $\lambda > 0$. Define $t := \varepsilon^2$; note that $t$ increases as contraction weakens (larger $\varepsilon$), so it suffices to show $R$ is strictly increasing in $t$. Setting $r_{\mathrm{te}} = 0$, the risk simplifies to $R(\varepsilon; 0) = (1 - a_c)^2 + a_c^2 \sigma_\eta^2 + a_s^2 \sigma_s^2$. We show each summand is strictly increasing in $t$.

For the spurious term: $a_s(t) = c_1 t / (c_2 + c_3 t)$ with $c_1 = r\sigma_s(\sigma_\eta^2 + \lambda)$, $c_2 = \lambda(1 + \sigma_\eta^2 + \lambda) > 0$, and $c_3 = \sigma_s^2(1 + \sigma_\eta^2 + \lambda - r^2) > 0$. Its derivative $c_1 c_2 / (c_2 + c_3 t)^2$ has sign $\operatorname{sign}(r)$ and never vanishes, so $a_s^2$ is strictly increasing in $t$.

For the core term, write $a_c(t) = N(t)/\Delta(t)$ where both $N(t)$ and $\Delta(t)$ are affine in $t$. Direct differentiation gives:
\begin{equation}
\begin{aligned}
a_c'(t)
&= \frac{N'(t)\,\Delta(t) - N(t)\,\Delta'(t)}{\Delta(t)^2}.
\end{aligned}
\end{equation}
Since both functions are affine, the numerator $N'\Delta - N\Delta' = n_1 d_0 - n_0 d_1$ is constant in $t$. Substituting $n_0 = \lambda$, $n_1 = \sigma_s^2(1-r^2)$, $d_0 = \lambda(1+\sigma_\eta^2+\lambda)$, and $d_1 = \sigma_s^2(1+\sigma_\eta^2+\lambda-r^2)$:
\begin{equation}
\begin{aligned}
n_1 d_0 - n_0 d_1
&= \sigma_s^2(1-r^2)\,\lambda(1+\sigma_\eta^2+\lambda) - \lambda\sigma_s^2(1+\sigma_\eta^2+\lambda-r^2) \\
&= -\lambda \sigma_s^2 r^2(\sigma_\eta^2 + \lambda) < 0.
\end{aligned}
\end{equation}
Since $\Delta(t)^2 > 0$, we conclude $a_c'(t) < 0$, so $a_c$ is strictly decreasing in $t$. Since $a_c < 1/(1 + \sigma_\eta^2) =: a^*$ (the test-optimal core coefficient), and $f(a) := (1-a)^2 + \sigma_\eta^2 a^2$ has $f'(a) < 0$ for $a < a^*$, the composition $f(a_c(t))$ is strictly increasing in $t$. Summing the two strictly increasing terms shows $R(\varepsilon; 0)$ is strictly increasing in $t = \varepsilon^2$, hence strictly decreasing as $\varepsilon$ decreases (stronger contraction).

\textbf{Claim (ii): Regularization is necessary.}
As $\lambda \to 0^+$, $\Delta_\varepsilon \to \varepsilon^2 \sigma_s^2(1 + \sigma_\eta^2 - r^2)$ and the factors of $\varepsilon^2$ cancel in both expressions of~\eqref{eq:coefficients_app}:
\begin{equation}
\begin{aligned}
a_c &\to \frac{1 - r^2}{1 + \sigma_\eta^2 - r^2}, \\
a_s &\to \frac{r\sigma_\eta^2}{\sigma_s(1 + \sigma_\eta^2 - r^2)}.
\end{aligned}
\end{equation}
Both coefficients, hence the entire fitted predictor, become independent of $\varepsilon$. This confirms that without regularization, the downstream probe can freely rescale its weight along the contracted direction by $1/\varepsilon$, undoing the contraction entirely.

\textbf{Claim (iii): Alignment counterexample.}
When $u_E = u_Y$, the representation collapses to one dimension $h = z_c + \eta + \sigma_s z_e$ (taking $r = 0$ for clarity, $\sigma_\eta^2 = \sigma_s^2 = 1$, $\lambda = 0.1$). The contraction rescales this single coordinate by $\varepsilon$. The ridge coefficient on the unshrunk feature is:
\begin{equation}
v(\varepsilon) = \frac{\varepsilon^2}{3\varepsilon^2 + 0.1},
\end{equation}
giving:
\begin{equation}
R(\varepsilon) = 1 - 2v(\varepsilon) + 3v(\varepsilon)^2.
\end{equation}
Numerically, $R(1) \approx 0.667$ while $R(0.1) \approx 0.864$: contraction strictly hurts because the contracted direction carries the core signal. This confirms that the separation between $u_E$ and $u_Y$ is essential for the mechanism to succeed. \hfill$\blacksquare$

\clearpage

\section{DFR Protocol}
\label{app:dfr_protocol}

We follow Deep Feature Reweighting (DFR;~\citealt{kirichenko2023last}), which retrains only the last layer of a fixed encoder and thereby reweights the features the encoder already exposes. Given an encoder checkpoint $f_{\theta}$, the protocol is:
\begin{enumerate}[leftmargin=*,itemsep=1pt,topsep=2pt]
\item \textbf{Freeze and extract.} Freeze $f_{\theta}$, discard its original classification head, and extract pooled representations $\mbf{h}_G = f_{\theta}(G)$ for the training, validation, and test graphs.
\item \textbf{Group-balanced reweighting set.} Define groups by the pair (label, environment). Keep all graphs of the smallest group and randomly subsample every other group to the same size, yielding a balanced set $\hat{\mathcal{D}}$ drawn from the group-labeled validation split; group labels are used only here and never during encoder training.
\item \textbf{Standardize.} Standardize each latent dimension with the mean and standard deviation computed on $\hat{\mathcal{D}}$, and apply the same transform to validation and test representations.
\item \textbf{Sparse linear head.} Fit a logistic-regression head on $\hat{\mathcal{D}}$ with $\ell_1$ regularization, $\min_{W,b}\sum_{(\mbf{h},y)\in\hat{\mathcal{D}}}\ell(W\mbf{h}+b,\,y)+\lambda\|W\|_1$. The sparsity penalty suits the regime $|\hat{\mathcal{D}}| \ll d$ and pushes the coefficients of irrelevant or spurious directions toward zero. The inverse strength $C \propto 1/\lambda$ is selected from $\{1, 0.7, 0.3, 0.1, 0.07, 0.03, 0.01\}$ by worst-group accuracy on a held-out half of the validation split, after which the head is refit on the full balanced set.
\item \textbf{Average over resamples.} Steps 2--4 are repeated $10$ times with independently resampled balanced subsets, and the resulting weights and biases are averaged, $W_{\mathrm{DFR}} = \tfrac{1}{10}\sum_{r=1}^{10} W^{(r)}$, which stabilizes the head.
\end{enumerate}
The final predictor is $\hat{y} = \arg\max W_{\mathrm{DFR}} f_{\theta}(G) + b_{\mathrm{DFR}}$, evaluated on the OOD test split. In \pcref{fig:dfr_verification}, $f_{\theta}$ is a GIN encoder taken from checkpoints at epochs $\{10, 30, 60, 90, 120, 150\}$ of either supervised ERM training or unsupervised VGAE training on GOOD-Motif concept shift; the \emph{ERM} curve reports the accuracy of the ERM model with its original head. All numbers are means and standard deviations over 3 seeds. DFR does not modify the representation: it changes only how the frozen features are weighted, so its accuracy measures what the encoder already makes linearly accessible.

\section{Additional Real-World Results}
\label{app:realworld_details}

The real-world experiments freeze the learned graph encoders and fit linear probes using 5\%, 10\%, or 20\% of the labeled training split. Model and configuration selection uses validation performance, and the test split is reserved for final evaluation. This section expands the representative analyses in the main text across all available molecular shifts.

\subsection{Hyperparameter Sensitivity across DrugOOD}
\label{app:realworld_hparams}

Tables~\ref{tab:realworld_hparam_top6_5}, \ref{tab:realworld_hparam_top6_10}, and~\ref{tab:realworld_hparam_top6_20} list the six highest-scoring evaluated $(\alpha,\beta,K)$ configurations for each of the six DrugOOD settings at 5\%, 10\%, and 20\% labeled-training ratios, respectively. Within each dataset--domain block, rows are ordered by test ROC-AUC at that ratio; untested configurations are not included or interpolated.
\input{tables/realworld_hparam_top6}

\subsection{Complete Training Dynamics}
\label{app:realworld_losses}

Across the complete set of runs, the three objectives remain bounded, although their scales and transient behavior depend on the dataset and split. These plots are optimization diagnostics: stable losses rule out persistent numerical divergence but do not by themselves establish invariance or causal removal of environment information.

\input{tables/realworld_loss_figures}

\clearpage
\section{AI Usage Statement}
\label{app:ai_usage}
The authors used generative AI tools to polish the writing, to assist with the literature survey, to assist with the implementation of experiments, and to check mathematical derivations. All research ideas, method design, experimental setups, results, and conclusions were developed and verified by the authors, who take full responsibility for the content of this paper.

%% file: tables/realworld_hparam_top6.tex
\begin{table*}[t]
\centering
\small
\caption{\textbf{Top evaluated pseudo-environment configurations at 5\% labels.} Within each dataset--domain setting, the six displayed $(\alpha,\beta,K)$ configurations are ordered from highest to lowest test ROC-AUC (\%). Untested configurations are not interpolated.}
\label{tab:realworld_hparam_top6_5}
\setlength{\tabcolsep}{5.5pt}
\renewcommand{\arraystretch}{1.06}
\begin{tabular}{llrrcr}
\toprule
Dataset & Domain & $\alpha$ & $\beta$ & $K$ & Test ROC-AUC (\%) \\
\midrule
EC50 & Assay & 0 & 10 & 3 & 69.84 \\
 &  & 0 & 10 & 20 & 65.70 \\
 &  & 0.5 & 10 & 20 & 65.63 \\
 &  & 0 & 10 & 5 & 65.61 \\
 &  & 0.5 & 10 & 3 & 64.28 \\
 &  & 1 & 10 & 20 & 63.96 \\
 & Scaffold & 1 & 10 & 20 & 63.75 \\
 &  & 0 & 10 & 20 & 61.71 \\
 &  & 0 & 0.01 & 20 & 61.00 \\
 &  & 0 & 10 & 5 & 60.30 \\
 &  & 0 & 1 & 20 & 60.27 \\
 &  & 0 & 0 & 20 & 50.0 \\
 & Size & 1 & 10 & 20 & 58.83 \\
 &  & 1 & 10 & 10 & 58.65 \\
 &  & 0.5 & 10 & 3 & 58.59 \\
 &  & 0 & 10 & 3 & 58.24 \\
 &  & 0 & 10 & 10 & 57.32 \\
 &  & 0 & 10 & 5 & 57.14 \\
\midrule
IC50 & Assay & 1 & 10 & 20 & 65.70 \\
 &  & 0 & 10 & 5 & 64.87 \\
 &  & 0.5 & 1 & 3 & 64.74 \\
 &  & 1 & 10 & 3 & 64.38 \\
 &  & 0.5 & 10 & 3 & 64.32 \\
 &  & 0 & 0.1 & 3 & 64.24 \\
 & Scaffold & 0 & 10 & 3 & 60.57 \\
 &  & 0.5 & 10 & 3 & 60.40 \\
 &  & 0 & 10 & 5 & 59.46 \\
 &  & 0.5 & 10 & 5 & 59.23 \\
 &  & 0 & 10 & 10 & 58.53 \\
 &  & 1 & 10 & 5 & 58.48 \\
 & Size & 0 & 10 & 20 & 58.22 \\
 &  & 1 & 10 & 20 & 58.14 \\
 &  & 0 & 10 & 5 & 58.11 \\
 &  & 0.5 & 10 & 3 & 57.83 \\
 &  & 0 & 10 & 3 & 57.71 \\
 &  & 0.5 & 1 & 20 & 57.55 \\
\bottomrule
\end{tabular}
\end{table*}

\begin{table*}[t]
\centering
\small
\caption{\textbf{Top evaluated pseudo-environment configurations at 10\% labels.} Within each dataset--domain setting, the six displayed $(\alpha,\beta,K)$ configurations are ordered from highest to lowest test ROC-AUC (\%). Untested configurations are not interpolated.}
\label{tab:realworld_hparam_top6_10}
\setlength{\tabcolsep}{5.5pt}
\renewcommand{\arraystretch}{1.06}
\begin{tabular}{llrrcr}
\toprule
Dataset & Domain & $\alpha$ & $\beta$ & $K$ & Test ROC-AUC (\%) \\
\midrule
EC50 & Assay & 0 & 10 & 3 & 70.38 \\
 &  & 0 & 10 & 10 & 68.45 \\
 &  & 0 & 10 & 5 & 68.32 \\
 &  & 0.5 & 10 & 5 & 67.17 \\
 &  & 0 & 0.1 & 3 & 66.59 \\
 &  & 0.5 & 1 & 20 & 66.39 \\
 & Scaffold & 1 & 10 & 10 & 61.03 \\
 &  & 0 & 10 & 10 & 61.00 \\
 &  & 0.5 & 10 & 10 & 60.99 \\
 &  & 0 & 0.01 & 20 & 60.53 \\
 &  & 1 & 10 & 20 & 60.31 \\
 &  & 0 & 0 & 20 & 49.9 \\
 & Size & 1 & 10 & 20 & 60.90 \\
 &  & 0 & 10 & 3 & 60.06 \\
 &  & 0.5 & 10 & 3 & 59.62 \\
 &  & 0.5 & 1 & 20 & 59.16 \\
 &  & 1 & 10 & 10 & 59.11 \\
 &  & 0 & 1 & 20 & 59.06 \\
\midrule
IC50 & Assay & 1 & 10 & 20 & 66.90 \\
 &  & 0 & 0.1 & 3 & 65.77 \\
 &  & 0.5 & 1 & 3 & 65.48 \\
 &  & 0 & 10 & 20 & 65.16 \\
 &  & 1 & 10 & 3 & 65.00 \\
 &  & 0 & 10 & 5 & 64.88 \\
 & Scaffold & 0 & 10 & 3 & 61.52 \\
 &  & 1 & 10 & 5 & 61.11 \\
 &  & 0.5 & 10 & 3 & 60.94 \\
 &  & 1 & 10 & 20 & 60.84 \\
 &  & 0.5 & 10 & 10 & 60.83 \\
 &  & 0 & 10 & 20 & 60.55 \\
 & Size & 0 & 10 & 20 & 59.44 \\
 &  & 0.5 & 1 & 20 & 59.44 \\
 &  & 1 & 10 & 20 & 59.18 \\
 &  & 0.5 & 1 & 3 & 58.94 \\
 &  & 0.5 & 10 & 3 & 58.84 \\
 &  & 0 & 10 & 3 & 58.74 \\
\bottomrule
\end{tabular}
\end{table*}

\begin{table*}[t]
\centering
\small
\caption{\textbf{Top evaluated pseudo-environment configurations at 20\% labels.} Within each dataset--domain setting, the six displayed $(\alpha,\beta,K)$ configurations are ordered from highest to lowest test ROC-AUC (\%). Untested configurations are not interpolated.}
\label{tab:realworld_hparam_top6_20}
\setlength{\tabcolsep}{5.5pt}
\renewcommand{\arraystretch}{1.06}
\begin{tabular}{llrrcr}
\toprule
Dataset & Domain & $\alpha$ & $\beta$ & $K$ & Test ROC-AUC (\%) \\
\midrule
EC50 & Assay & 0 & 10 & 3 & 73.07 \\
 &  & 0.5 & 10 & 3 & 71.55 \\
 &  & 0 & 10 & 5 & 69.67 \\
 &  & 0.5 & 1 & 20 & 69.25 \\
 &  & 0 & 1 & 3 & 68.31 \\
 &  & 0.5 & 1 & 3 & 68.31 \\
 & Scaffold & 0 & 10 & 10 & 61.23 \\
 &  & 0.5 & 10 & 10 & 60.59 \\
 &  & 0 & 0.01 & 20 & 60.42 \\
 &  & 0 & 10 & 20 & 60.06 \\
 &  & 0 & 10 & 5 & 59.93 \\
 &  & 0 & 0 & 20 & 53.1 \\
 & Size & 0.5 & 10 & 3 & 61.61 \\
 &  & 0 & 10 & 3 & 60.71 \\
 &  & 1 & 10 & 20 & 60.71 \\
 &  & 0 & 10 & 20 & 59.84 \\
 &  & 0.5 & 10 & 5 & 59.76 \\
 &  & 0 & 1 & 3 & 59.53 \\
\midrule
IC50 & Assay & 1 & 10 & 20 & 67.39 \\
 &  & 0 & 0.1 & 3 & 66.60 \\
 &  & 0 & 10 & 20 & 66.29 \\
 &  & 0.5 & 1 & 3 & 66.22 \\
 &  & 1 & 10 & 3 & 65.56 \\
 &  & 0 & 10 & 5 & 65.25 \\
 & Scaffold & 0.5 & 10 & 3 & 61.92 \\
 &  & 0.5 & 10 & 10 & 61.74 \\
 &  & 0 & 10 & 3 & 61.54 \\
 &  & 0 & 10 & 5 & 61.47 \\
 &  & 0.5 & 10 & 5 & 61.46 \\
 &  & 0 & 10 & 20 & 61.41 \\
 & Size & 0.5 & 1 & 20 & 60.73 \\
 &  & 0 & 10 & 20 & 60.09 \\
 &  & 1 & 10 & 20 & 60.04 \\
 &  & 0.5 & 1 & 3 & 59.91 \\
 &  & 0 & 1 & 20 & 59.43 \\
 &  & 0 & 10 & 3 & 59.36 \\
\bottomrule
\end{tabular}
\end{table*}

\begin{table}[ht]
\centering
\large
\caption{\textbf{GOOD-Motif concept hyperparameter sensitivity.} The evaluated configurations use fixed $K=8$ and $k_{\rm adv}=2$. Test accuracy is the best-epoch linear-probe result.}
\label{tab:goodmotif_hparam_sensitivity}
\setlength{\tabcolsep}{12.0pt}
\renewcommand{\arraystretch}{1.25}
\begin{tabular}{rrr}
\toprule
$\alpha$ & $\beta$ & Test accuracy (\%) \\
\midrule
0 & 0.01 & 65.0 \\
0 & 0.1 & 62.8 \\
0 & 1 & 66.0 \\
0 & 10 & 61.7 \\
0.5 & 0.01 & 62.3 \\
0.5 & 0.1 & 69.6 \\
0.5 & 1 & 63.6 \\
0.5 & 10 & 53.8 \\
1.0 & 0.01 & 65.3 \\
1.0 & 0.1 & 69.6 \\
\bottomrule
\end{tabular}
\end{table}

%% file: tables/realworld_loss_figures.tex
\begin{figure*}[t]
\centering
\includegraphics[width=0.96\textwidth]{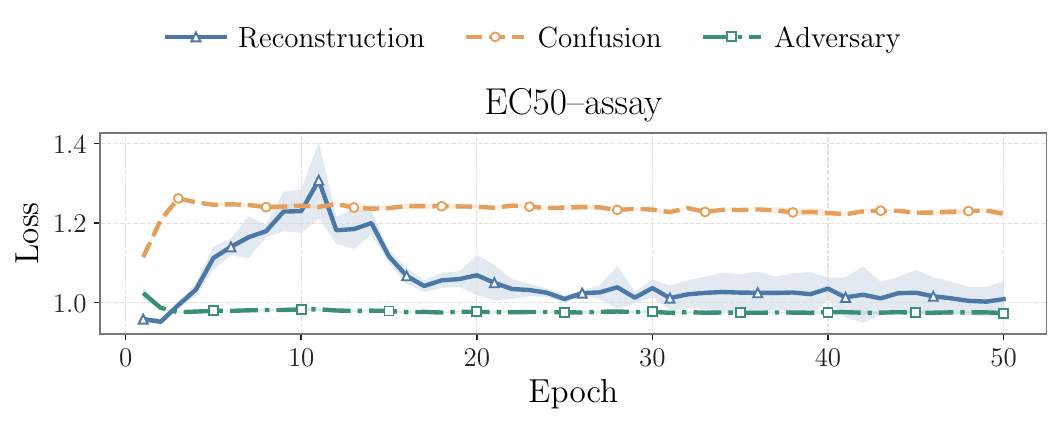}
\caption{\textbf{Training dynamics on EC50-assay.} The plot uses the validation-selected $(\alpha,\beta,K)=(0,10,3)$. Reconstruction, confusion, and adversary losses use solid, dashed, and dash-dotted lines with triangle, circle, and square markers, respectively. Lines are means over three seeds and shaded regions denote mean $\pm$ sample standard deviation.}
\label{fig:loss_ec50_assay}
\end{figure*}

\begin{figure*}[t]\centering
\includegraphics[width=0.96\textwidth]{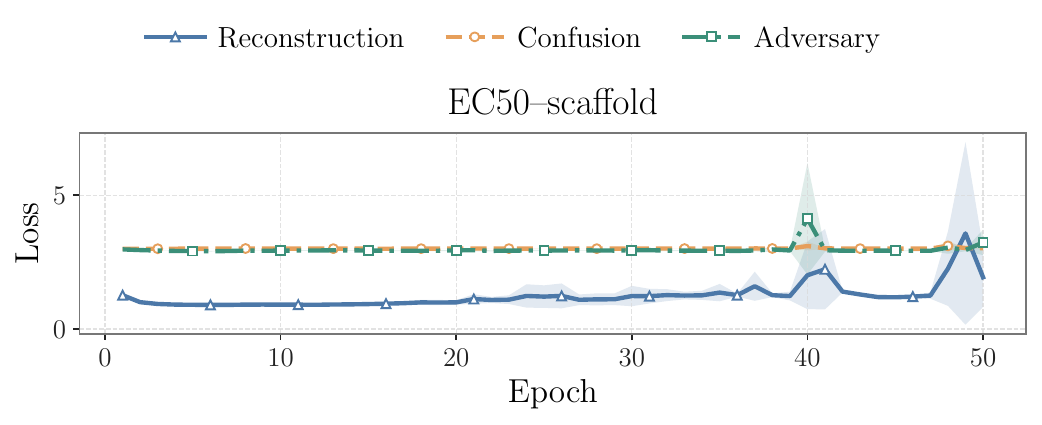}
\caption{\textbf{Training dynamics on EC50-scaffold.} The validation-selected $(\alpha,\beta,K)$ configuration is $(1,10,20)$. Lines and shaded regions follow the convention of Figure~\ref{fig:loss_ec50_assay}.}
\label{fig:loss_ec50_scaffold}\end{figure*}

\begin{figure*}[t]\centering
\includegraphics[width=0.96\textwidth]{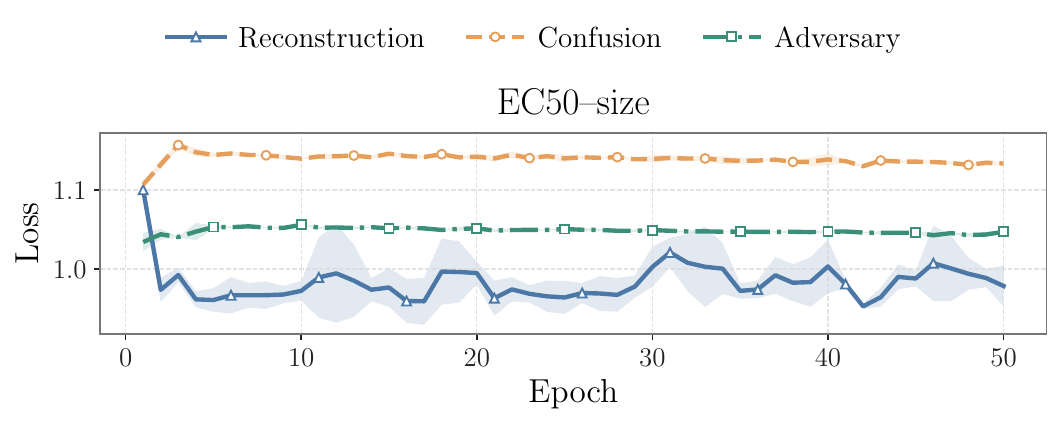}
\caption{\textbf{Training dynamics on EC50-size.} The validation-selected $(\alpha,\beta,K)$ configuration is $(0.5,10,3)$. Lines and shaded regions follow the convention of Figure~\ref{fig:loss_ec50_assay}.}
\label{fig:loss_ec50_size}\end{figure*}

\begin{figure*}[t]\centering
\includegraphics[width=0.96\textwidth]{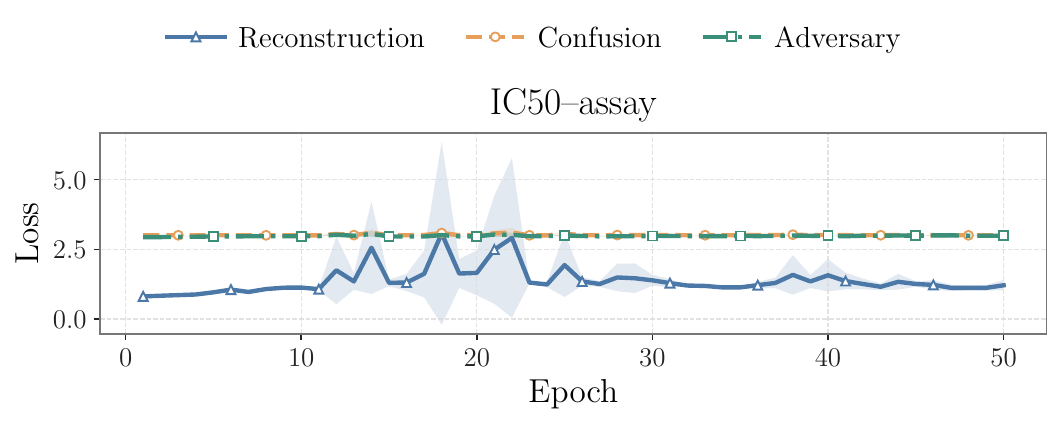}
\caption{\textbf{Training dynamics on IC50-assay.} The validation-selected $(\alpha,\beta,K)$ configuration is $(1,10,20)$. Lines and shaded regions follow the convention of Figure~\ref{fig:loss_ec50_assay}.}
\label{fig:loss_ic50_assay}\end{figure*}

\begin{figure*}[t]\centering
\includegraphics[width=0.96\textwidth]{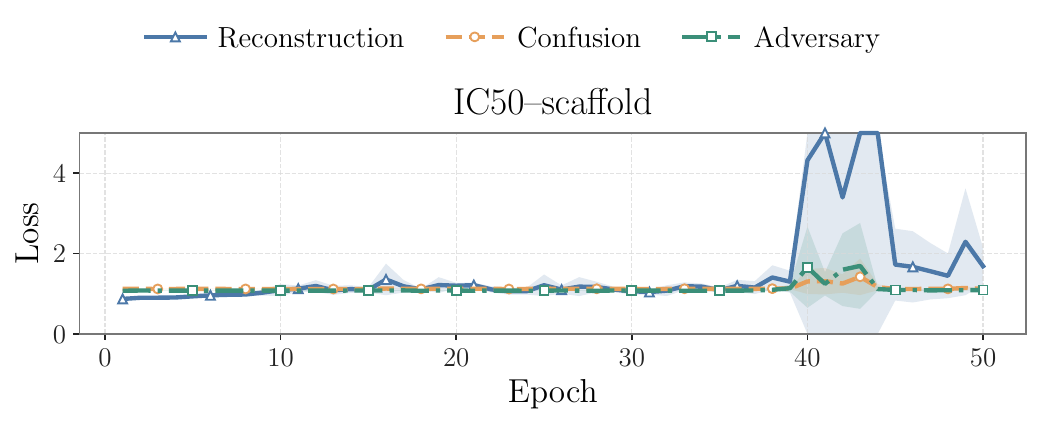}
\caption{\textbf{Training dynamics on IC50-scaffold.} The validation-selected $(\alpha,\beta,K)$ configuration is $(0.5,10,3)$. The vertical axis is truncated at 5 to show the stable-loss range; values above this limit are placed at the boundary and connected to the next sampled vertex. Lines and shaded regions follow the convention of Figure~\ref{fig:loss_ec50_assay}.}
\label{fig:loss_ic50_scaffold}\end{figure*}

\begin{figure*}[t]\centering
\includegraphics[width=0.96\textwidth]{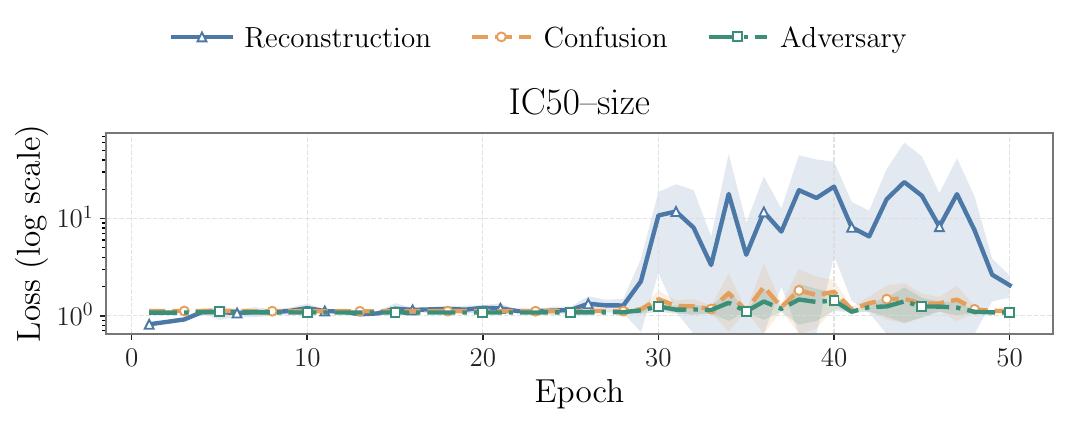}
\caption{\textbf{Training dynamics on IC50-size.} The validation-selected $(\alpha,\beta,K)$ configuration is $(0.5,10,3)$. The vertical axis is logarithmic because reconstruction has a substantially larger transient scale. Lines and shaded regions follow the convention of Figure~\ref{fig:loss_ec50_assay}.}
\label{fig:loss_ic50_size}\end{figure*}